%% file: main.tex
\documentclass[twoside]{article}

\usepackage[accepted]{aistats2020blank}
\usepackage{natbib}
\setcitestyle{authoryear,round,citesep={;},aysep={,},yysep={;}}

\usepackage[tableposition=top]{caption}
\usepackage{amsmath,amsthm,amssymb,mathtools,graphicx,enumitem,hyphenat,float, threeparttable}
\usepackage{mathabx, hyperref}
\hypersetup{ %
    pdfborder=0 0 0,
    pdfpagemode=UseNone,
    colorlinks=true,
    linkcolor=blue,
    citecolor=blue,
    filecolor=blue,
    urlcolor=blue,
    pdfview=FitH}
\usepackage{import, ifthen}
\usepackage{algorithm, algorithmic}
\usepackage[symbol]{footmisc}
\renewcommand{\|}{\parallel}

\newcommand\ec[2][]{\ensuremath{\mathbb{E}_{#1} \left[#2\right]}}
\newcommand\ecn[2][]{\ec[#1]{\norm{#2}^2}}

\newcommand\ev[1]{\left \langle #1 \right \rangle}
\newcommand\br[1]{\left ( #1 \right )}
\newcommand\pbr[1]{\left \{ #1 \right \} }
\newcommand\floor[1]{\left \lfloor #1 \right \rfloor}

\newcommand{\N}{\mathbb{N}}

\newcommand{\R}{\mathbb{R}}

\newcommand{\e}{\varepsilon}

\newcommand{\Om}{\Omega}

\newcommand{\mc}{\mathcal}

\newcommand{\norm}[1]{\left\lVert#1\right\rVert}
\newcommand{\sqn}[1]{{\left\lVert#1\right\rVert}^2}
\newcommand{\abs}[1]{\left\lvert#1\right\rvert}

\newcommand{\D}{\mathcal{D}}
\newcommand{\eqdef}{\overset{\text{def}}{=}}

\newcommand{\avemm}{\frac{1}{M}\sum_{m=1}^M}
\newcommand{\aveim}{\frac{1}{M}\sum_{i=1}^M}

\newcommand{\sigmaopt}{\sigma_{\mathrm{opt}}}
\newcommand{\sigmaf}{\sigma_{\mathrm{dif}}}
\makeatletter
\newsavebox\myboxA
\newsavebox\myboxB
\newlength\mylenA

\usepackage{tcolorbox}
\usepackage{pifont}
\definecolor{mydarkgreen}{RGB}{39,130,67}
\newcommand{\green}{\color{mydarkgreen}}
\definecolor{mydarkred}{RGB}{192,47,25}
\newcommand{\red}{\color{mydarkred}}
\newcommand{\cmark}{\green\ding{51}}%
\newcommand{\xmark}{\red\ding{55}}%

\newcommand*\overbar[2][0.75]{%
    \sbox{\myboxA}{$\m@th#2$}%
    \setbox\myboxB\null
    \ht\myboxB=\ht\myboxA%
    \dp\myboxBtdp\myboxA%
    \wd\myboxB=#1\wd\myboxA
    \sbox\myboxB{$\m@th\overline{\copy\myboxB}$}
    \setlength\mylenA{\the\wd\myboxA}
    \addtolength\mylenA{-\the\wd\myboxB}%
    \ifdim\wd\myboxB<\wd\myboxA%
       \rlap{\hskip 0.5\mylenA\usebox\myboxB}{\usebox\myboxA}%
    \else
        \hskip -0.5\mylenA\rlap{\usebox\myboxA}{\hskip 0.5\mylenA\usebox\myboxB}%
    \fi}
\makeatother

\theoremstyle{definition}
\newtheorem{theorem}{Theorem}
\newtheorem{corollary}{Corollary}

\newtheorem{proposition}{Proposition}

\newtheorem{asm}{Assumption}

\newtheorem{lemma}{Lemma}

\theoremstyle{definition}

\usepackage[colorinlistoftodos,bordercolor=orange,backgroundcolor=orange!20,linecolor=orange,textsize=scriptsize]{todonotes}

\begin{document}

\twocolumn[

\aistatstitle{Tighter Theory for Local SGD  on Identical and Heterogeneous Data\footnotemark[1] }

\runningtitle{Tighter Theory for Local SGD on Identical and Heterogeneous Data}

\aistatsauthor{Ahmed Khaled \And Konstantin Mishchenko \And Peter Richt\'arik }

\aistatsaddress{ Cairo University, KAUST\footnotemark[2] \And  KAUST \And KAUST } ]

\footnotetext[1]{This work extends two papers \cite{khaled2019analysis,khaled2019better} presented at the NeurIPS 2019 Federated Learning Workshop. }
\footnotetext[2]{This paper was prepared when the author was a research intern at KAUST.}

\begin{abstract}
     We provide a new analysis of local SGD, removing unnecessary assumptions and elaborating on the difference between two data regimes: identical and heterogeneous. In both cases, we improve the existing theory and provide values of the optimal stepsize and optimal number of local iterations. Our bounds are based on a new notion of variance that is specific to local SGD methods with different data. The tightness of our results is guaranteed by recovering known statements when we plug $H=1$, where $H$ is the number of local steps. The empirical evidence further validates the severe impact of data  heterogeneity on the  performance of local  SGD.
\end{abstract}

\section{Introduction}
Modern hardware increasingly relies on the power of uniting many parallel units into one system. This approach requires optimization methods that target specific issues arising in distributed environments such as decentralized data storage. Not having data in one place implies that computing nodes have to communicate back and forth to keep moving toward the solution of the overall problem. A number of efficient first-, second-order and dual methods that are capable of reducing the communication overhead existed in the literature for a long time, some of which are in certain sense optimal.

Yet, when Federated Learning (FL) showed up, it turned out that the problem of balancing the communication and computation had not been solved. On the one hand, Minibatch Stochastic Gradient Descent (SGD), which averages the result of stochastic gradient steps computed in parallel on several machines, again demonstrated its computation efficiency. Seeking communication efficiency, \cite{Konecny16, McMahan17} proposed to use a natural variant of Minibatch SGD---\emph{Local SGD} (Algorithm~\ref{alg:local_sgd}), which does a few SGD iterations locally on each involved node and only then computes the average. This approach saves a lot of time on communication, but, unfortunately, in terms of theory things were not as great as in terms of practice and there are still gaps in our understanding of Local SGD. 

The idea of local SGD in fact is not recent, it traces back to the work of~\cite{Mangasarian95} and has since been popular among practitioners from different communities. An asymptotic analysis can be found in~\cite{Mangasarian95} and quite a few recent papers proved new convergence results, making the bounds tighter with every work. The theory has been developing in two important regimes: identical and heterogeneous data.

The identical data regime is more of interest if the data are actually stored in one place. In that case, we can access it on each computing device at no extra cost and get a fast, scalable method. Although not very general, this framework is already of interest to a wide audience due to its efficiency in training large-scale machine learning models~\citep{TaoLin19}. The first contribution of our work is to provide the fastest known rate of convergence for this regime under weaker assumptions than in prior work.

Federated learning, however, is done on a very large number of mobile devices, and is operating in a highly non-i.i.d.\ regime. To address this, we present the first analysis of Local SGD that applies to \emph{arbitrarily heterogeneous data}, while all previous works assumed a certain type of similarity between the data or local gradients.

\import{algs/}{gen_loc_sgd_alg}

To explain the challenge of heterogeneity better, let us introduce the problem we are trying to solve. Given that there are $M$ devices and corresponding local losses $f_m:\R^d\to \R$, we want to find
\begin{equation}
     \label{eq:optimization-problem}
     \min_{x \in \R^d} \pbr{ f(x) = \frac{1}{M} \sum_{m=1}^{M} f_m (x)}.
\end{equation}
In the case of identical data, we are able to obtain on each node an unbiased estimate of the gradient $\nabla f$. In the case of heterogeneous data, $m$-th node can only obtain an unbiased estimate of the gradient $\nabla f_m$. Data similarity can then be formulated in terms of the differences between functions $f_1,\dotsc, f_M$. If the underlying data giving rise to the loss functions are i.i.d., the function share optima and one could even minimize them separately, averaging the results at the end. We will demonstrate this rigorously later in the paper.

If the data are dissimilar, however, we need to be much more careful since running SGD locally will yield solutions of local problems. Clearly, their average might not minimize the true objective~\eqref{eq:optimization-problem}, and this poses significant issues for the convergence of Local SGD. 

To properly discuss the efficiency of local SGD, we also need a practical way of quantifying it. Normally, a method's efficiency is measured by the total number of times each function $f_m$ is touched and the cost of the touches. On the other hand, in distributed learning we also care about how much information each computing node needs to communicate. In fact, when communication is as expensive as is the case in FL, we predominantly care about communication. The question we address in this paper, thus, can be posed as follows: how many times does each node need to communicate if we want to solve~\eqref{eq:optimization-problem} up to accuracy $\e$? Equivalently, we can ask for the optimal \textit{synchronization interval length} between communications, $H$, i.e.\ how many computation steps per one communication we can allow for. We next review related work and then present our contributions.
\begin{table*}[h]
    \caption{Existing theoretical bounds for local SGD for \textbf{identical data} with convex objectives.}
    \begin{threeparttable}[b]
     \centering
     \def\arraystretch{1.2}
     \resizebox{\textwidth}{!}{
     \begin{tabular}{|c|c|c|c|c|c|c|c|c|}
     \hline
          \begin{tabular}{c}Unbounded\\ gradient \end{tabular} & \begin{tabular}{c} $H=T$ \\ convergent \end{tabular} & \begin{tabular}{c} $C(T)$\tnote{a} \\ $f$ strongly convex \end{tabular} & \begin{tabular}{c} $C(T)$ \\ $f$ convex \end{tabular} & Reference \\
     \hline    
          \xmark & \xmark &  $\Om(\sqrt{ M T })$ & \xmark & \citeauthor{Stich2018}, 5/2018 \\
     \hline
          \xmark & \xmark & $\Om\br{\sqrt{M T}}$ & \xmark & \citeauthor{Basu2019}, 6/2018 \\
     \hline
          \cmark & \xmark & $\tilde{\Om}\br{M}$ & $\Om(M^{3/2} T^{1/2})$ & \citeauthor{stich19errorfeedback}, 9/2019 \\
     \hline
          \cmark & \xmark & $\tilde{\Om}\br{M^{1/3} T^{1/3} }$\tnote{b} & - & \citeauthor{Haddadpour2019local}, 10/2019 \\
     \hline
          \cmark & \cmark & $\tilde{\Om}(M)$ & $\Om(M^{3/2} T^{1/2})$ & \textbf{THIS WORK, 9/2019-1/2020} \\
     \hline
     \end{tabular}
     }
     \begin{tablenotes}
         \item [a] $C(T)$ denotes the minimum number of communication steps required at each of $T$ iterations to achieve a linear speedup in the number of nodes $M$.
         \item [b] The PL inequality, a generalization of strong convexity, is assumed in \citep{Haddadpour2019local}, but for comparison we specialize to strong convexity.
     \end{tablenotes}
    \end{threeparttable}
    \label{tab:related_work_convexity}
\end{table*}

\section{Related Work}
While local SGD has been used among practitioners for a long time, see e.g.~\citep{Coppola15, McDonald10}, its theoretical analysis has been limited until recently. Early theoretical work on the convergence of local methods exists as in~\citep{Mangasarian95}, but no convergence rate was given there. The previous work can mainly be divided into two groups: those assuming identical data (that all nodes have access to the same dataset) and those that allow each node to hold its own dataset. As might be expected, the analysis in the latter case is more challenging, more limited, and usually shows worse rates. We note that in recent work more sophisticated local stochastic gradient methods have been considered, for example with momentum~\citep{yu2019linear, wang2019slowmo}, with quantization~\citep{Reisizadeh19,Basu2019}, with adaptive stepsizes~\citep{xie2019local} and with various variance-reduction methods \citep{liang2019variance, sharma2019parallel, karimireddy2019scaffold}. Our work is complimentary to these approaches, and provides improved rates and analysis for the vanilla method.

\renewcommand{\thefootnote}{\arabic{footnote}}
\subsection{Local SGD with Identical Data}
The analysis of local SGD in this setting shows that a reduction in communication is possible without affecting the asymptotic convergence rate of Minibatch SGD with $M$ nodes (albeit with usually worse dependence on constants). An overview of related work on local SGD for convex objectives is given in Table~\ref{tab:related_work_convexity}. We note that analysis for nonconvex objectives has been carried out in a few recent works~\citep{Zhou18,Wang18,Jiang18}, but our focus in this work is on convex objectives and hence they were not included in Table~\ref{tab:related_work_convexity}. The comparison shows that we attain superior rates in the strongly convex setting to previous work with the exception of the concurrent\footnote[1]{Made available online one day after the first version of our work was.} work of~\cite{stich19errorfeedback} and we attain these rates under less restrictive assumptions on the optimization process compared to them. We further provide a novel analysis in the convex case, which has not been previously explored in the literature, with the exception of \citep{stich19errorfeedback}. Their analysis attains the same communication complexity but is much more pessimistic about possible values of $H$. In particular, it does not recover the convergence of one-shot averaging, i.e.\ substituting $H = T$ or even $H = T/M$ gives noninformative bounds, unlike our Theorem~\ref{thm:sc-convergence-theorem}.

In addition to the works listed in the table, \cite{Patel19}  also analyze local SGD for identical data under a Hessian smoothness assumption in addition to gradient smoothness, strong convexity, and uniformly bounded variance. However, we believe that there are issues in their proof that we explain in Section~\ref{sec:patel-discuss} in the supplementary material. As a result, the work is excluded from the table. 

\subsection{Local SGD with Heterogeneous Data}
\begin{table*}[h]
    \caption{Existing theoretical bounds for local SGD with \textbf{heterogeneous data}.}
    \centering
    \def\arraystretch{1.4}
    \resizebox{\textwidth}{!}{
    \begin{tabular}{|c|c|c|c|c|c|c|c|c|c|c|}
    \hline
         \begin{tabular}{c}Unbounded\\ gradient \end{tabular} & \begin{tabular}{c}Unbounded\\ dissimilarity/diversity \end{tabular} & \begin{tabular}{c} $C(T)$ \\ $f$ strongly convex \end{tabular} & \begin{tabular}{c} $C(T)$ \\ $f$ convex \end{tabular} & \begin{tabular}{c} $C(T)$ \\ $f$ nonconvex \end{tabular} & Reference \\
    \hline
         \xmark & \xmark & - & - & $\Om\br{M^{3/4} T^{3/4}}$ & \citeauthor{Yu18}, 7/2018 \\
    \hline   
         \cmark & \xmark & - & - & $\Om(T)$ & \citeauthor{Jiang18}, 12/2018  \\
    \hline 
         \xmark & \xmark & $\Om\br{\sqrt{MT}}$ & - & $\Om\br{M^{3/4} T^{3/4}}$ & \citeauthor{Basu2019}, 6/2019  \\
    \hline
         \cmark & \xmark & $\Om\br{M^{1/3} T^{1/3}}$ & - & $\Om\br{M^{3/2} T^{1/2}}$ & \citeauthor{Haddadpour19FL}, 10/2019   \\
    \hline
         \cmark & \cmark & - & $\Om\br{M^{3/4} T^{3/4}}$ & - & \textbf{THIS WORK, 1/2020} \\
    \hline
    \end{tabular}
    }
    \label{tab:related_work_non-iid}
\end{table*}

An overview of related work on local SGD in this setting is given in Table~\ref{tab:related_work_non-iid}. In addition to the works in Table~\ref{tab:related_work_non-iid}, \cite{WangTuor18} analyze a local gradient descent method under convexity, bounded dissimilarity, and bounded gradients, but do not show convergence to arbitrary precisions. \cite{Li2019} analyze federated averaging (discussed below) in the strongly convex and nonconvex cases under bounded gradient norms. However, their result is not included in Table~\ref{tab:related_work_non-iid} because in the more general setting of federated averaging, their analysis and experiments suggest that retaining a linear speedup is not possible. 

Local SGD is at the core of the {\em Federated Averaging} algorithm which is popular in federated learning applications~\citep{Konecny16}. Essentially, Federated Averaging is a variant of Local SGD with participating devices sampled randomly. This algorithm has been used in several machine learning applications such as mobile keyboard prediction~\citep{HardRao18}, and strategies for improving its communication efficiency were explored in~\citep{Konecny16}. Despite its empirical success, little is known about convergence properties of this method and it has been observed to diverge when too many local steps are performed~\citep{McMahan17}. This is not so surprising as the majority of common assumptions are not satisfied; in particular, the data are typically very non-i.i.d.~\citep{McMahan17}, so the local gradients can point in different directions. This property of the data can be written for any vector $x$ and indices $i,j$ as
\begin{align*}
	\norm{\nabla f_i(x) - \nabla f_j(x)} \gg 1.
\end{align*}
Unfortunately, it is very hard to analyze local methods without assuming a bound on the dissimilarity of $\nabla f_i(x)$ and $\nabla f_j(x)$. For this reason, almost all prior work assumed some regularity notion over the functions such as bounded dissimilarity \citep{yu2019linear, Li2019, Yu18, WangTuor18} or bounded gradient diversity \citep{Haddadpour19FL} and addressed other less challenging aspects of federated learning such as decentralized communication, nonconvexity of the objective or unbalanced data partitioning. In fact, a common way to make the analysis simple is to assume Lipschitzness of local functions,
$
	\norm{\nabla f_i(x)} \le G
$
for any $x$ and $i$. We argue that this assumption is pathological and should be avoided when seeking a meaningful convergence bound. First of all, in unconstrained strongly convex minimization this assumption can not be satisfied, making the analysis in works like~\citep{Stich2018} questionable. Second, there exists at least one method, whose convergence is guaranteed under bounded variance~\citep{juditsky2011solving}, but in practice the method diverges~\citep{chavdarova2019reducing, mishchenko2019revisiting}. Finally, under the bounded gradients assumption we have
\begin{align*}
    \label{eq:related-work-2}
    \norm{\nabla f_i (x) - \nabla f_{j} (x)} \leq \norm{\nabla f_i (x)} + \norm{\nabla f_j (x)} \leq 2G.
\end{align*}
In other words, we lose control over the difference between the functions. Since $G$ bounds not just dissimilarity, but also the gradients themselves, it makes the statements less insightful or even vacuous. For instance, it is not going to be tight if the data are actually i.i.d.\ since $G$ in that case will remain a positive constant. In contrast, we will show that the rate should depend on a much more meaningful quantity, 
$$
	\sigmaf^2 \eqdef \frac{1}{M} \sum_{m=1}^{M} \ecn[z_m \sim \D_m]{\nabla f_m (x_*, z_m)},
$$
where $x_*$ is a fixed minimizer of $f$ and $f_m (\cdot, z_m)$ for $z_m \sim \D$ are stochastic realizations of $f_m$ (see the next section for the setting). Obviously, for all nondegenerate sampling distributions $\D_m$ the quantity $\sigmaf$ is finite and serves as a natural measure of variance in local methods. We note that an attempt to get more general convergence statement has been made by~\citep{Sahu18}, but unfortunately their guarantee is strictly worse than that of minibatch Stochastic Gradient Descent (SGD). In the overparameterized regime where $\sigmaf = 0$, \cite{zhang2019distributed} prove the convergence of Local SGD with arbitrary $H$.

Our earlier workshop paper~\citep{khaled2019analysis} explicitly analyzed Local Gradient Descent (Local GD) as opposed to Local SGD, where there is no stochasticity in the gradients. An analysis of Local GD for non-convex objectives with the PL inequality and under bounded gradient diversity was subsequently carried out by~\cite{Haddadpour19FL}. 

\section{Settings and Contributions}
\begin{asm}
     \label{asm:convexity-and-smoothness}
     Assume that the set of minimizers of \eqref{eq:optimization-problem} is nonempty. Each $f_m$ is $\mu$-strongly convex for $\mu \geq 0$ and $L$-smooth. That is, for all $x, y \in \R^d$
     \begin{align*}
          \frac{\mu}{2} \sqn{x - y} &\leq f_m (x) - f_m (y) - \ev{\nabla f_m (y), x - y} \\
          &\leq \frac{L}{2} \sqn{x - y}.
     \end{align*}
     When $\mu = 0$, we say that each $f_m$ is just convex. When $\mu \neq 0$, we define $\kappa \eqdef \frac{L}{\mu}$, the condition number.
\end{asm}

Assumption~\ref{asm:convexity-and-smoothness} formulates our requirements on the overall objective. Next, we have two different sets of assumptions on the stochastic gradients that model different scenarios, which also lead to different convergence rates.

\begin{asm}
     \label{asm:uniformly-bounded-variance}
     Given a function $h$, a point $x \in \R^d$, and a sample $z \sim \D$ drawn i.i.d.\ according to a distribution $\D$, the stochastic gradients $g = g(h, x, z)$ satisfy
     $
          \ec[z \sim \D]{g(h, x, z)} = \nabla h (x),
          \ecn[z \sim \D]{g(h, x, z) - \nabla h(x)} \leq \sigma^2.
     $
\end{asm}

Assumption~\ref{asm:uniformly-bounded-variance} holds for example when $g(x, z) = \nabla h(x) + \xi_z$ for a random variable $\xi_z$ of expected bounded squared norm: $\ecn[z \sim \D]{\xi_z} \leq \sigma^2$. Assumption~\ref{asm:uniformly-bounded-variance}, however, typically does not hold for finite-sum problems where $g(x, z)$ is a gradient of the one functions in the finite-sum. To capture this setting, we consider the following assumption:

\begin{asm}
    \label{asm:finite-sum-stochastic-gradients}
    Given an $L$-smooth and $\mu$-strongly convex (possibly with $\mu = 0$) function $h: \R^d \to \R$ written as an expectation $h = \ec[z \sim \D]{h(x, z)}$, we assume that a stochastic gradient $g = g(h, x, z)$ is computed by $g(h, x, z) = \nabla h(x, z).$
    We assume that $h(\cdot, z): \R^d \to \R$ is almost-surely $L$-smooth and $\mu$-strongly convex (with the same $L$ and $\mu$ as $h$).
\end{asm}

When Assumption~\ref{asm:finite-sum-stochastic-gradients} is assumed \textbf{in the identical data setting}, we assume it is satisfied on each node $m \in [M]$ with $h = f$ and distribution $\D_m$, and we define as a measure of variance at the optimum
\[ \sigmaopt^2 \eqdef \frac{1}{M} \sum_{m=1}^{M} \ecn[z_m \sim \D_m]{\nabla f (x_\ast, z_m)}. \] 
Whereas \textbf{in the heterogeneous data} setting we assume that it is satisfied on each node $m \in [M]$ with $h = f_m$ and distribution $\D_m$, and we analogously define
\[ \sigmaf^2 \eqdef \frac{1}{M} \sum_{m=1}^{M} \ecn[z_m \sim \D_m]{\nabla f_m (x_\ast, z_m)}. \]

Assumption~\ref{asm:finite-sum-stochastic-gradients} holds, for example, for finite-sum optimization problems with uniform sampling and permits direct extensions to more general settings such as expected smoothness \cite{SGD-AS}.

\textbf{Our contributions} are as follows:
\begin{enumerate}
     \item In the identical data setting under Assumptions~\ref{asm:convexity-and-smoothness} and~\ref{asm:uniformly-bounded-variance} with $\mu > 0$, we prove that the iteration complexity of Local SGD to achieve $\e$-accuracy is 
     \[ \mathcal{\tilde{O}}\br{ \frac{\sigma^2}{\mu^2 M \e} } \]
     in squared distance from the optimum provided that $T = \Omega\br{ \kappa \br{H  - 1} }$. This improves the communication complexity in prior work (see Table~\ref{tab:related_work_convexity}) with a tighter results compared to concurrent work (recovering convergence for $H = 1$ and $H = T$). When $\mu = 0$ we have that the iteration complexity of Minibatch SGD to attain an $\e$-accurate solution in functional suboptimality is 
     \[ \mathcal{O} \br{\ \frac{L^2 \norm{x_0 - x_\ast}^4}{M \e^2} + \frac{\sigma^4}{L^2 M \e^2}}, \] 
     provided that $T = \Omega\br{M^3 H^2}$. We further show that the same $\e$-dependence holds in both the $\mu > 0$ and $\mu = 0$ cases under Assumption~\ref{asm:finite-sum-stochastic-gradients}. This has not been explored in the literature on Local SGD before, and hence we obtain the first results that apply to arbitrary convex and smooth finite-sum problems. 
     \item When the data on each node is different and Assumptions~\ref{asm:convexity-and-smoothness} and \ref{asm:finite-sum-stochastic-gradients} hold with $\mu = 0$, the iteration complexity needed by Local SGD to achieve an $\e$-accurate solution in functional suboptimality is 
     \[ \mathcal{O} \br{ \frac{L^2 \norm{x_0 - x_\ast}^4}{M \e^2} + \frac{\sigmaf^4}{L^2 M \e^2}} \]
     provided that $T = \Omega(M^3 H^4)$. This improves upon previous work by not requiring any restrictive assumptions on the gradients and is the first analysis to capture true \emph{data heterogeneity} between different nodes.
     \item We verify our results by experimenting with logistic regression on multiple datasets, and investigate the effect of heterogeneity on the convergence speed.
\end{enumerate}

\section{Convergence Theory}

The following quantity is crucial to the analysis of both variants of local SGD, and measures the deviation of the iterates from their average $\hat{x}_t$ over an epoch:
\[ V_t \eqdef \avemm \sqn{x_t^m - \hat{x}_t} \text { where } \hat{x}_t \eqdef \avemm x_t^m. \]
To prove our results, we follow the line of work started by \cite{Stich2018} and first show a recurrence similar to that of SGD up to an error term proportional to $V_t$, then we bound each $V_t$ term individually or the sum of $V_t$'s over an epoch. All proofs are relegated to the supplementary material.

\subsection{Identical Data}
Our first lemma presents a bound on the sequence of the $V_t$ in terms of the synchronization interval $H$.
\begin{lemma}
     \label{lemma:uniform-var-iterate-variance-bound}
     Choose a stepsize $\gamma > 0$ such that $\gamma \leq \frac{1}{2L}$. Under Assumptions~\ref{asm:convexity-and-smoothness}, and \ref{asm:uniformly-bounded-variance} we have that for Algorithm~\ref{alg:local_sgd} with $\max_{p} \abs{t_p - t_{p+1}} \leq H$ and with identical data, for all $t \ge 1$
     \begin{align*}
         \ec{V_{t}} \leq \br{H - 1} \gamma^2 \sigma^2.
     \end{align*}
 \end{lemma}

Combining Lemma~\ref{lemma:uniform-var-iterate-variance-bound} with perturbed iterate analysis as in \citep{Stich2018} we can recover the convergence of local SGD for strongly-convex functions:
 \begin{theorem}
     \label{thm:sc-convergence-theorem}
     Suppose that Assumptions~\ref{asm:convexity-and-smoothness}, and \ref{asm:uniformly-bounded-variance} hold with $\mu > 0$. Then for Algorithm~\ref{alg:local_sgd} run with identical data, a constant stepsize $\gamma > 0$ such that $\gamma \leq \frac{1}{4L}$, and $H \geq 1$ such that $\max_{p} \abs{t_p - t_{p+1}} \leq H$,
     \begin{align}
         \label{eq:thm-sc-convergence-rate}
         \begin{split}
          \ecn{\hat{x}_T - x_\ast} \leq (1 &- \gamma \mu)^{T} \sqn{x_0 - x_\ast} + \frac{\gamma \sigma^2}{\mu M} \\
          &+ \frac{2 L \gamma^2 \br{H - 1} \sigma^2}{\mu}.
         \end{split}
     \end{align}
 \end{theorem}

By \eqref{eq:thm-sc-convergence-rate} we see that the convergence of local SGD is the same as Minibatch SGD plus an additive error term which can be controlled by controlling the size of $H$, as the next corollary and the successive discussion show.

\begin{corollary}
    \label{corrollary:sc-convergence-comm-complexity}
     Choosing $\gamma = \frac{1}{\mu a}$, with $a = {4 \kappa + t}$ for $t > 0$ and we take $T = 2 a \log a$ steps. Then substituting in \eqref{eq:thm-sc-convergence-rate} and using that $1 - x \leq \exp(-x)$ and some algebraic manipulation we can conclude that,
     \begin{equation*}
        \ecn{r_T} = \tilde{\mathcal{O}} \br{ \frac{\sqn{r_0}}{T^2} + \frac{\sigma^2}{\mu^2 M T} + \frac{\kappa \sigma^2 (H-1)}{\mu^2 T^2}}.
     \end{equation*}
    where $r_{t} = \hat{x}_t - x_\ast$ and $\tilde{\mathcal{O}} (\cdot)$ ignores polylogarithmic and constant numerical factors. 
\end{corollary}

\textbf{Recovering fully synchronized Minibatch SGD.} When $H = 1$ the error term vanishes and we obtain directly the ordinary rate of Minibatch SGD. 

\textbf{Linear speedup in the number of nodes $M$.} We see that choosing $H = \mathcal{O}(T/M)$ leads to an asymptotic convergence rate of $\mathcal{\tilde{O}}\br{ \frac{\sigma^2 \kappa}{\mu^2 M T} }$ which shows the same linear speedup of Minibatch SGD but with worse dependence on $\kappa$. The number of communications in this case is then $C(T) = T/H = \tilde{\Omega}(M)$.

\textbf{Local SGD vs Minibatch SGD.} We assume that the statistical $\sigma^2/T$ dependence dominates the dependence on the initial distance $\sqn{x_0 - x_\ast} / T^2$. From Corollary~\ref{corrollary:sc-convergence-comm-complexity}, we see that in order to achieve the same convergence guarantees as Minibatch SGD, we must have $H = \mathcal{O} \br{\frac{T}{\kappa M}}$, achieving a communication complexity of $\mathcal{O} \br{\kappa M}$. This is only possible when $T > \kappa M$. It follows that given a number of steps $T$ the optimal $H$ is $H = 1 + \floor{T/(\kappa M)}$ achieving a communication complexity of $\tilde{\Omega} \br{ \min (T, \kappa M) }$. 

\textbf{One-shot averaging.} Putting $H = T + 1$ yields a convergence rate of $\mathcal{\tilde{O}}(\sigma^2 \kappa /(\mu^2 T))$, showing no linear speedup but showing convergence, which improves upon all previous work. However, we admit that simply using Jensen's inequality to bound the distance of the average iterate $\ecn{\hat{x}_T - x_\ast}$ would yield a better asymptotic convergence rate of $\mathcal{\tilde{O}}(\sigma^2 / (\mu^2 T))$. Under a Lipschitz Hessian assumption, \cite{Zhang13} show that one-shot averaging can attain a linear speedup in the number of nodes, so one may do analysis of local SGD under this additional assumption to try to remove this gap, but this is beyond the scope of our work.

Similar results can be obtained for weakly convex functions, as the next Theorem shows.

\begin{theorem}
     \label{thm:weakly-convex-thm}
     Suppose that Assumptions~\ref{asm:convexity-and-smoothness}, \ref{asm:uniformly-bounded-variance} hold with $\mu = 0$ and that a constant stepsize $\gamma$ such that $\gamma \geq 0$ and $\gamma \leq \frac{1}{4 L}$ is chosen and that Algorithm~\ref{alg:local_sgd} is run for identical data with $H \geq 1$ such that $\sup_{p} \abs{t_p - t_{p+1}} \leq H$, then for $\bar{x}_T = \frac{1}{T} \sum_{t=1}^{T} \hat{x}_t$,
     \begin{align}
         \label{eq:thm-wc-convergence-rate}
         \begin{split}
               \ec{f(\bar{x}_T) - f(x_\ast)} \leq \frac{2}{\gamma T} &\norm{x_0 - x_\ast}^2  + \frac{2 \gamma \sigma^2}{M} \\
               &+ 4 \gamma^2 L \sigma^2 \br{H - 1}.
         \end{split}
     \end{align}
\end{theorem}

Theorem~\ref{thm:weakly-convex-thm} essentially tells the same story as Theorem~\ref{thm:sc-convergence-theorem}: convergence of local SGD is the same as Minibatch SGD up to an additive constant whose size can be controlled by controlling $H$.
 
\begin{corollary}
    \label{corollary:weakly-convex-convergence-iid}
     Assume that $T \geq M$. Choosing $\gamma = \frac{\sqrt{M}}{4 L \sqrt{T}}$, then substituting in \eqref{eq:thm-wc-convergence-rate} we have,
     \begin{align*}
          \begin{split}
               \ec{f(\bar{x}_T) - f(x_\ast)} \leq &\frac{8 \sqn{x_0 - x_\ast}}{\sqrt{M T}} + \frac{\sigma^2}{2 L \sqrt{M T}} \\
               &+ \frac{\sigma^2 M \br{H - 1}}{L T}.
          \end{split}
     \end{align*}
\end{corollary}

\textbf{Linear speedup and optimal $H$.} From Corollary~\ref{corollary:weakly-convex-convergence-iid} we see that if we choose $H = \mc{O}(\sqrt{T} M^{-3/2})$ then we obtain a linear speedup, and the number of communication steps is then $C = T/H = \Omega\br{M^{3/2} T^{1/2}}$, and we get that the optimal $H$ is then $H = 1 + \floor{ T^{1/2} M^{-3/2} }$.

The previous results were obtained under Assumption~\ref{asm:uniformly-bounded-variance}. Unfortunately, this assumption does not easily capture the finite-sum minimization scenario where $f(x) = \frac{1}{n} \sum_{i=1}^{n} f_i (x)$ and each stochastic gradient $g_t$ is sampled uniformly at random from the sum. 

Using smaller stepsizes and more involved proof techniques, we can show that our results still hold in the finite-sum setting. For strongly-convex functions, the next theorem shows that the same convergence guarantee as Theorem~\ref{thm:sc-convergence-theorem} can be attained.
\begin{figure*}[t]
\centering
	\includegraphics[scale=0.23]{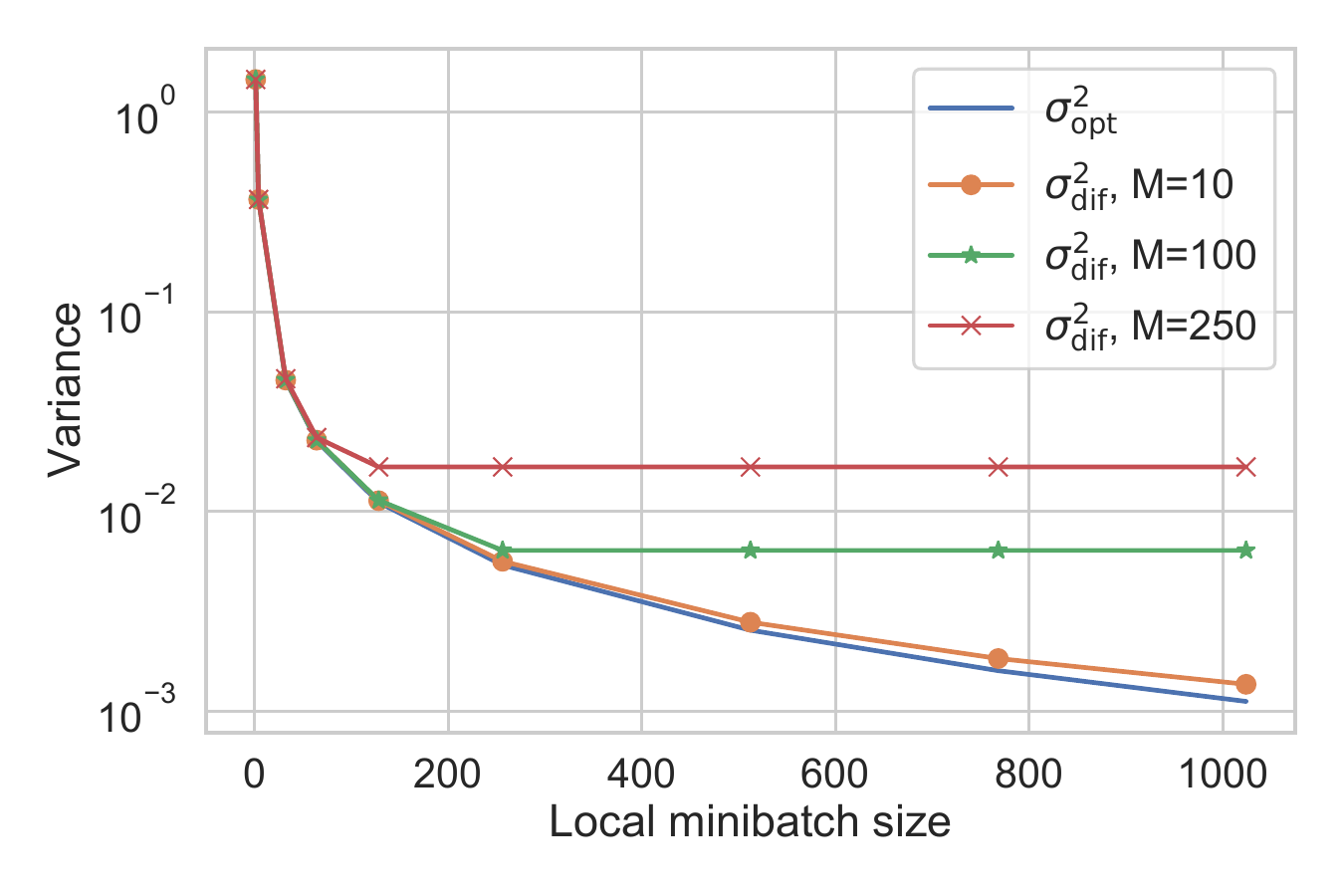}
	\includegraphics[scale=0.23]{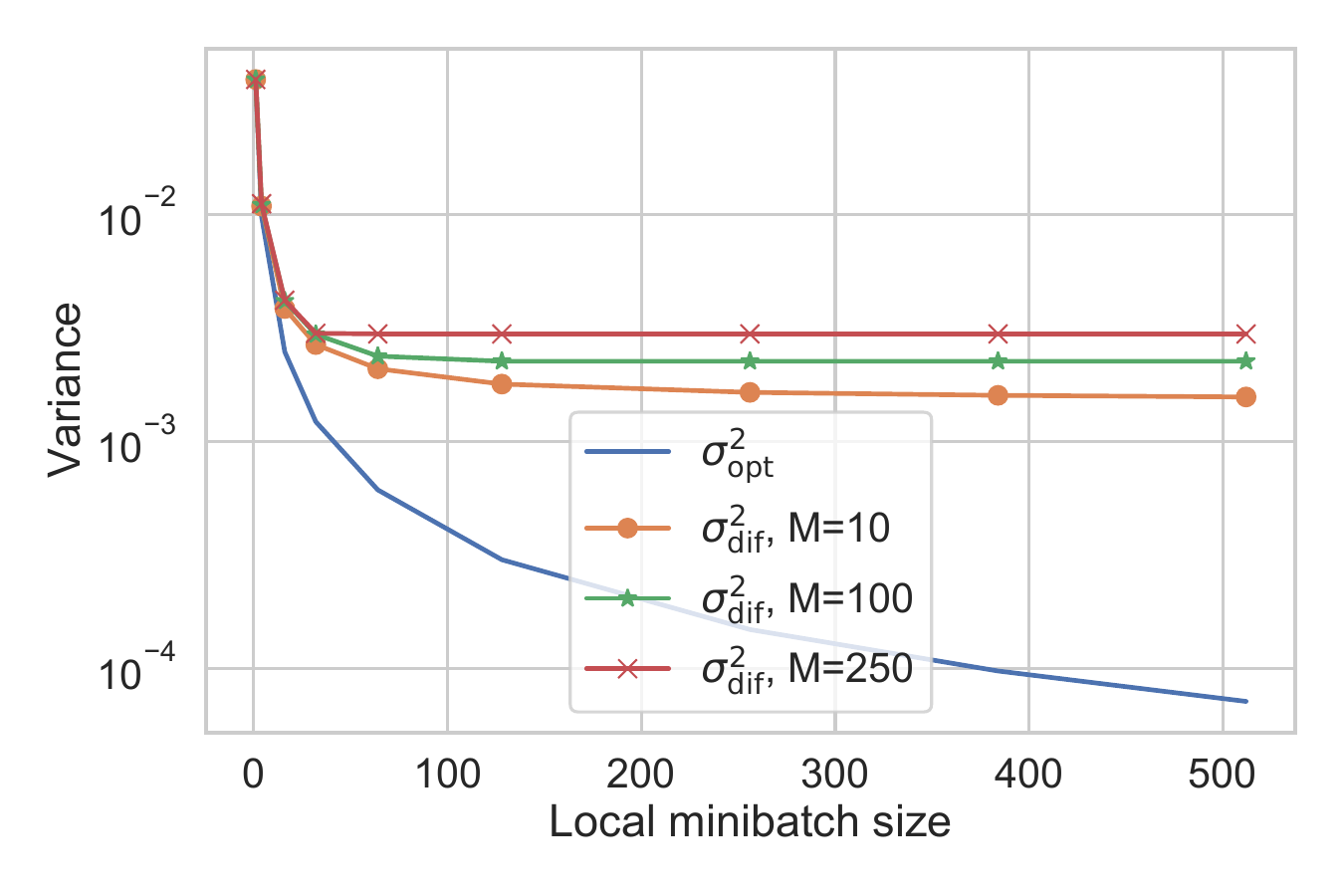}
	\includegraphics[scale=0.23]{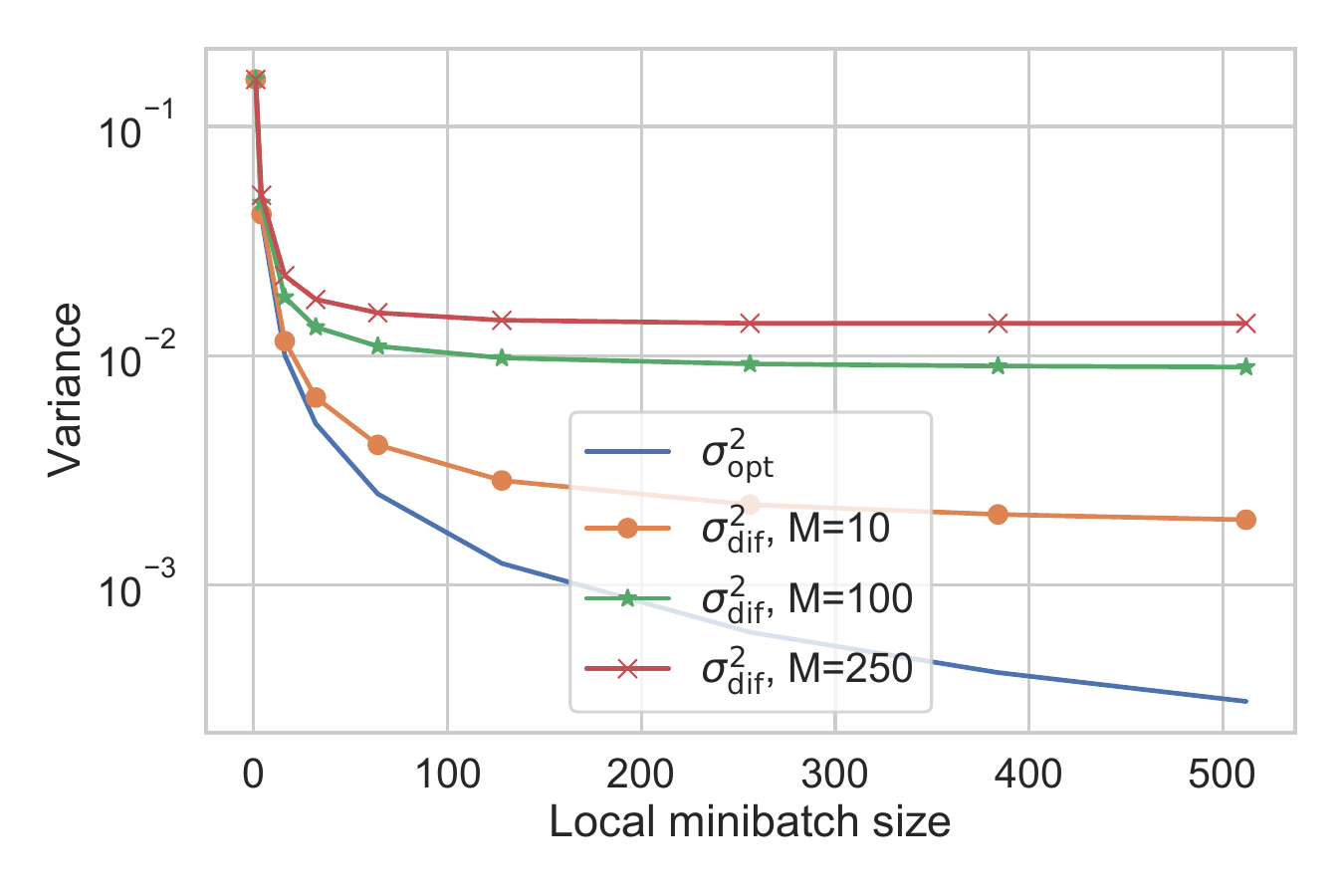}
	\caption{The effect of the dataset and number of workers $M$ on the variance parameters. Left: `a8a', middle: `mushrooms', right: `w8a' dataset. We use uniform sampling of data points, so $\sigmaopt^2$ is the same as $\sigmaf^2$ with $M=1$, while for higher values of $M$ the value of $\sigmaf^2$ might be drastically larger than $\sigmaopt^2$.}
	\label{fig:variance}
\end{figure*}
\begin{theorem}
    \label{theorem:sc-unbounded-variance-iid}
    Suppose that Assumptions~\ref{asm:convexity-and-smoothness} and \ref{asm:finite-sum-stochastic-gradients} hold with $\mu > 0$. Suppose that Algorithm~\ref{alg:local_sgd} is run for identical data with $\max_{p} \abs{t_{p} - t_{p+1}} \leq H$ for some $H \geq 1$ and with a stepsize $\gamma > 0$ chosen such that 
    $ \gamma \leq \min \pbr{\frac{1}{4L \br{1 + \frac{2}{M}}}, \frac{1}{\mu + 8 L \br{H-1}} }. $
    Then for any timestep $t$ such that synchronization occurs,
    \begin{align}
        \label{eq:thm-sc-const-gamma-unbounded-variance}
        \begin{split}
            \ecn{\hat{x}_t - x_\ast} &\leq \br{1 - \gamma \mu}^{t} \ecn{x_0 - x_\ast} \\
            &+ \frac{2 \gamma \sigmaopt^2}{\mu M} + \frac{4 \sigmaopt^2 \gamma^2 \br{H - 1} L}{\mu}.
        \end{split}
    \end{align}
\end{theorem}

As a corollary, we can obtain an asymptotic convergence rate by choosing specific stepsizes $\gamma$ and $H$.

\begin{corollary}
    \label{corollary:unbounded-var-convergence}
    Let $a = 18 \kappa t$ for some $t > 0$, let $H \leq t$ and choose $\gamma = \frac{1}{\mu a} \leq \frac{1}{9 L H}$. We substitute in \eqref{eq:thm-sc-const-gamma-unbounded-variance} and take $T = 18 a \log a$ steps, then for $r_{t} \eqdef \hat{x}_t - x_\ast$,
    \begin{align*}
           \ecn{r_{t}} = \mathcal{\tilde{O}} \Bigg ( \frac{\sqn{r_{0}}}{T^2} &+ \frac{\sigmaopt^2}{\mu^2 M T} + \frac{\sigmaopt^2 \kappa (H-1)}{\mu^2 T^2} \Bigg ).
    \end{align*}
\end{corollary}

Substituting $H = 1 + \floor{t/M} = 1 + \floor{T/(18 \kappa M)}$ in Corollary~\ref{corollary:unbounded-var-convergence} we get an asymptotic convergence rate of $\mc{\tilde{O}} \br{ \frac{\sigmaopt^2}{T M} }$. This preserves the rate of minibatch SGD up to problem-independent constants and polylogarithmic factors, but with possibly fewer communication steps.

\begin{theorem}
    \label{thm:wc-iid-unbounded-var}
    Suppose that Assumptions~\ref{asm:convexity-and-smoothness} and \ref{asm:finite-sum-stochastic-gradients} hold with $\mu = 0$, that a stepsize $\gamma \leq \frac{1}{10 L H}$ is chosen and that Algorithm~\ref{alg:local_sgd} is run on $M \geq 2$ nodes with identical data and with $\sup_{p} \abs{t_p - t_{p+1}} \leq H$, then for any timestep $T$ such that synchronization occurs we have for $\bar{x}_T = \frac{1}{T} \sum_{t=1}^{T} \hat{x}_t$ that
    \begin{align}
        \label{eq:thm-weakly-convex-case-iid-unbounded-var}
        \begin{split}
            \ec{ f(\bar{x}_T) - f(x_\ast) } \leq &\frac{10 \sqn{x_0 - x_\ast} }{\gamma T} + \frac{20 \gamma \sigmaopt^2}{M} \\
            &+ 40 \gamma^2 L \sigmaopt^2 \br{H-1}.
        \end{split}
    \end{align}
\end{theorem}

\begin{corollary}
    \label{corollary:wc-iid-conv-unbounded-var}
    Let $H \leq \frac{\sqrt{T}}{\sqrt{M}}$, then for $\gamma = \frac{\sqrt{M}}{10 L \sqrt{T}}$ we see that $\gamma \leq \frac{1}{10 L H}$, and plugging it into \eqref{eq:thm-weakly-convex-case-iid-unbounded-var} yields 
    \begin{align*}
        \begin{split}
            \ec{f(\bar{x}_{T}) - f(x_\ast)} \leq &\frac{100 L \sqn{x_0 - x_\ast}}{\sqrt{T M}} + \frac{2 \sigmaopt^2}{L \sqrt{T M}} \\
            &+ \frac{2 \sigmaopt^2 M (H-1)}{5 L T}.
        \end{split}
    \end{align*}
\end{corollary}

This is the same result as Corollary~\ref{corollary:weakly-convex-convergence-iid}, and hence we see that choosing $H = \mathcal{O}\br{T^{1/2} M^{-3/2}}$ (when $T > M^3$) yields a linear speedup in the number of nodes $M$.

\subsection{Heterogeneous Data}
We next show that for arbitrarily heterogeneous convex objectives, the convergence of Local SGD is the same as Minibatch SGD plus an error that depends on $H$.

\begin{theorem}
    \label{thm:wc-noniid-unbounded-var}
    Suppose that Assumptions~\ref{asm:convexity-and-smoothness} and \ref{asm:finite-sum-stochastic-gradients} hold with $\mu = 0$ and for heterogeneous data. Then for Algorithm~\ref{alg:local_sgd} run for different data with $M \geq 2$, $\max_{p} \abs{t_p - t_{p+1}} \leq H$, and a stepsize $\gamma > 0$ such that $\gamma \leq \min \pbr{ \frac{1}{4L} \frac{1}{8 L (H-1)}}$, then we have
    \begin{align*}
        \begin{split}
            \ec{f(\bar{x}_T) -f(x_\ast)} \leq &\frac{4 \sqn{r_0}}{\gamma T} + \frac{20 \gamma \sigmaf^2}{M} \\
            &+ 16 \gamma^2 L (H-1)^2 \sigmaf^2.
        \end{split}
    \end{align*}
    where $\bar{x}_T \eqdef \frac{1}{T} \sum_{i=0}^{T-1} \hat{x}_{i}$ and $r_{0} = x_0 - x_\ast$.
\end{theorem}

\textbf{Dependence on $\sigmaf$.} We see that the convergence guarantee given by Theorem~\ref{thm:wc-noniid-unbounded-var} shows a dependence on $\sigmaf$, which measures the heterogeneity of the data distribution. In typical (non-federated) distributed learning settings where data is distributed before starting training, this term can very quite significantly depending on how the data is distributed.

\textbf{Dependence on $H$.} We further note that the dependence on $H$ in Theorem~\ref{thm:wc-noniid-unbounded-var} is quadratic rather than linear. This translates to a worse upper bound on the synchronization interval $H$ that still allows convergence, as the next corollary shows.

\begin{figure*}[t]
	\centering
	\includegraphics[scale=0.23]{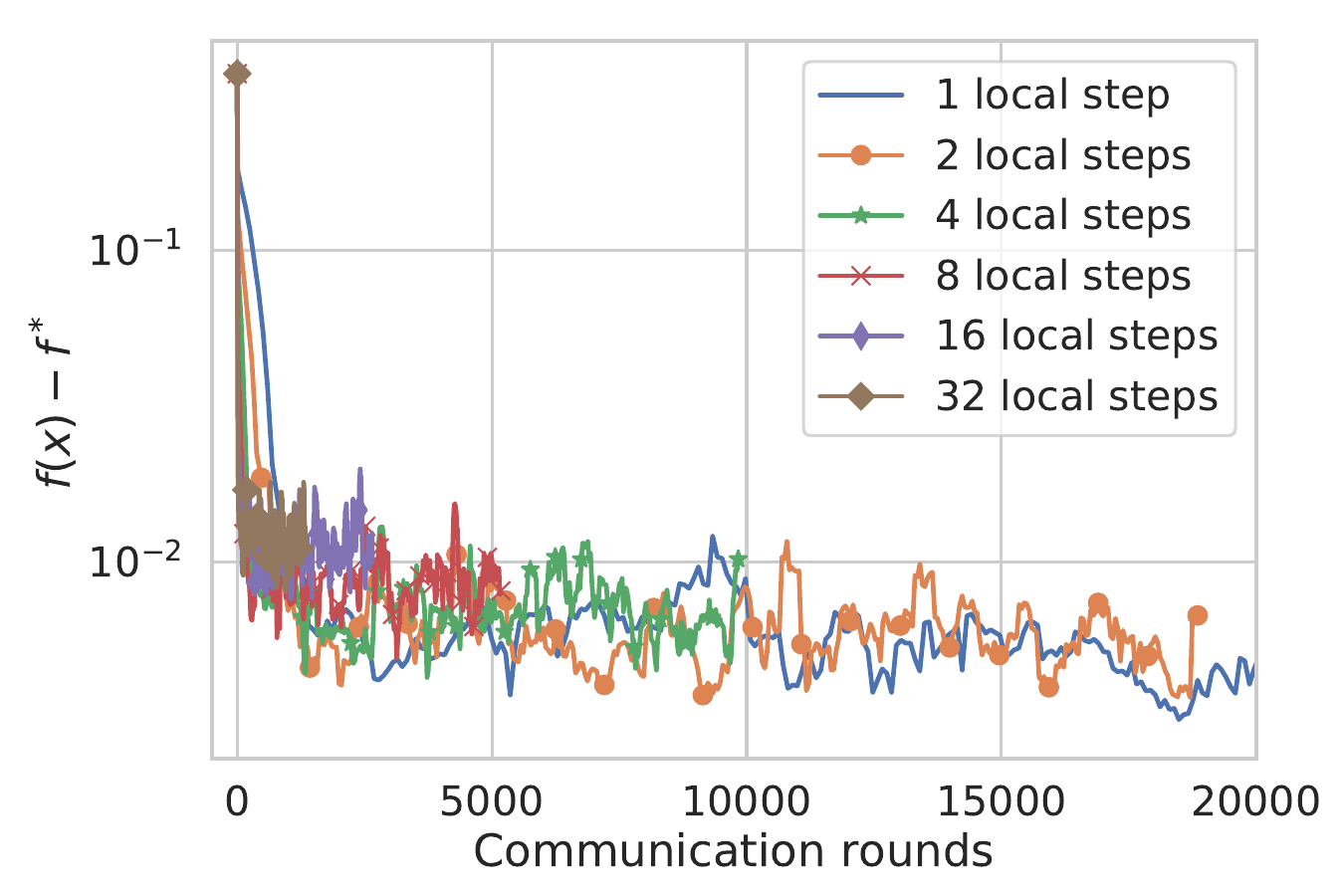}
	\includegraphics[scale=0.23]{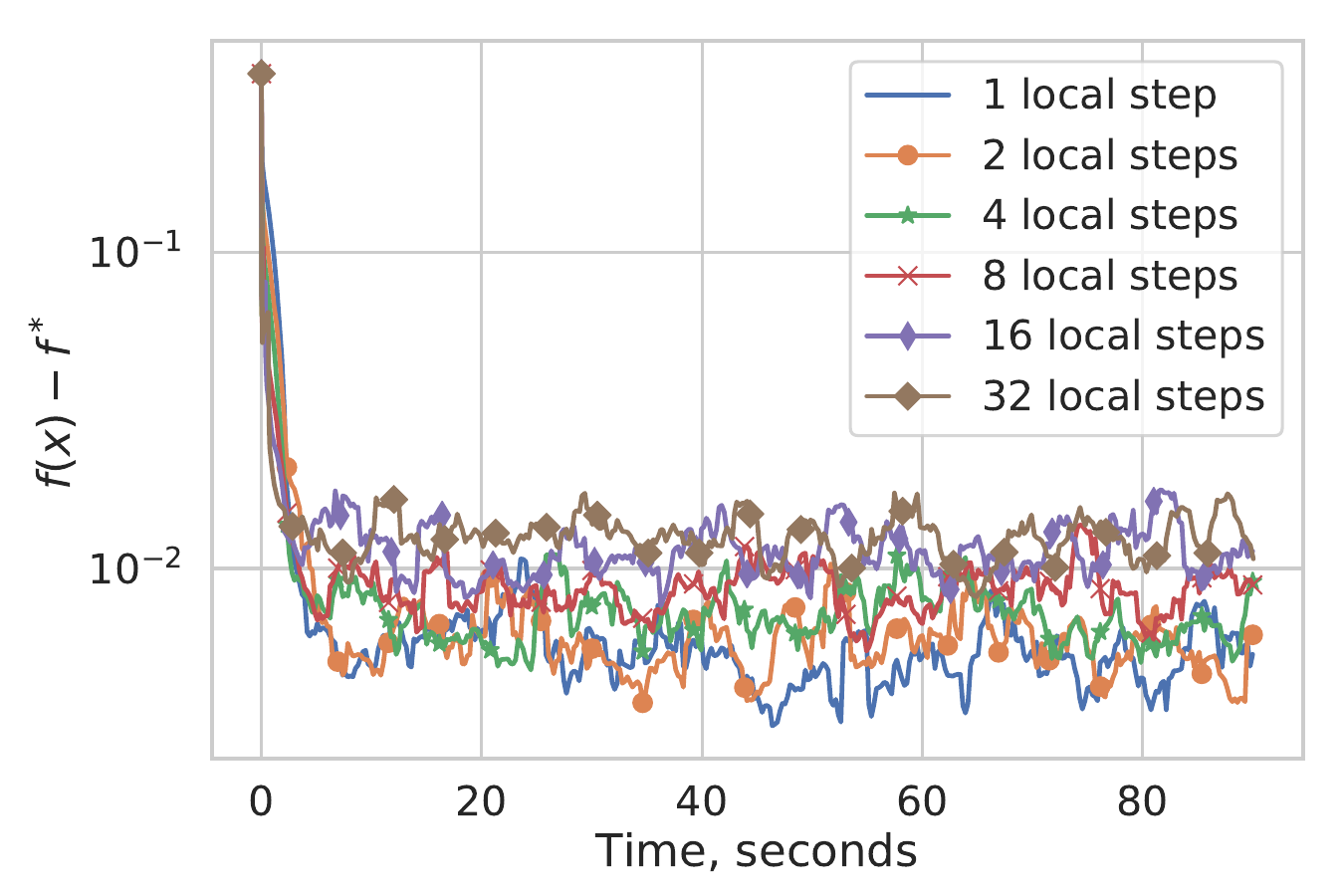}
	\caption{Results on `a9a' dataset, with stepsize $\frac{1}{L}$. For any value of local iterations $H$ the method converged to a neighborhood within a small number of communication rounds due to large stepsizes.}
	\label{fig:a5a_same_data_01}
\end{figure*}
\begin{figure*}[!t]
	\centering
	\includegraphics[scale=0.23]{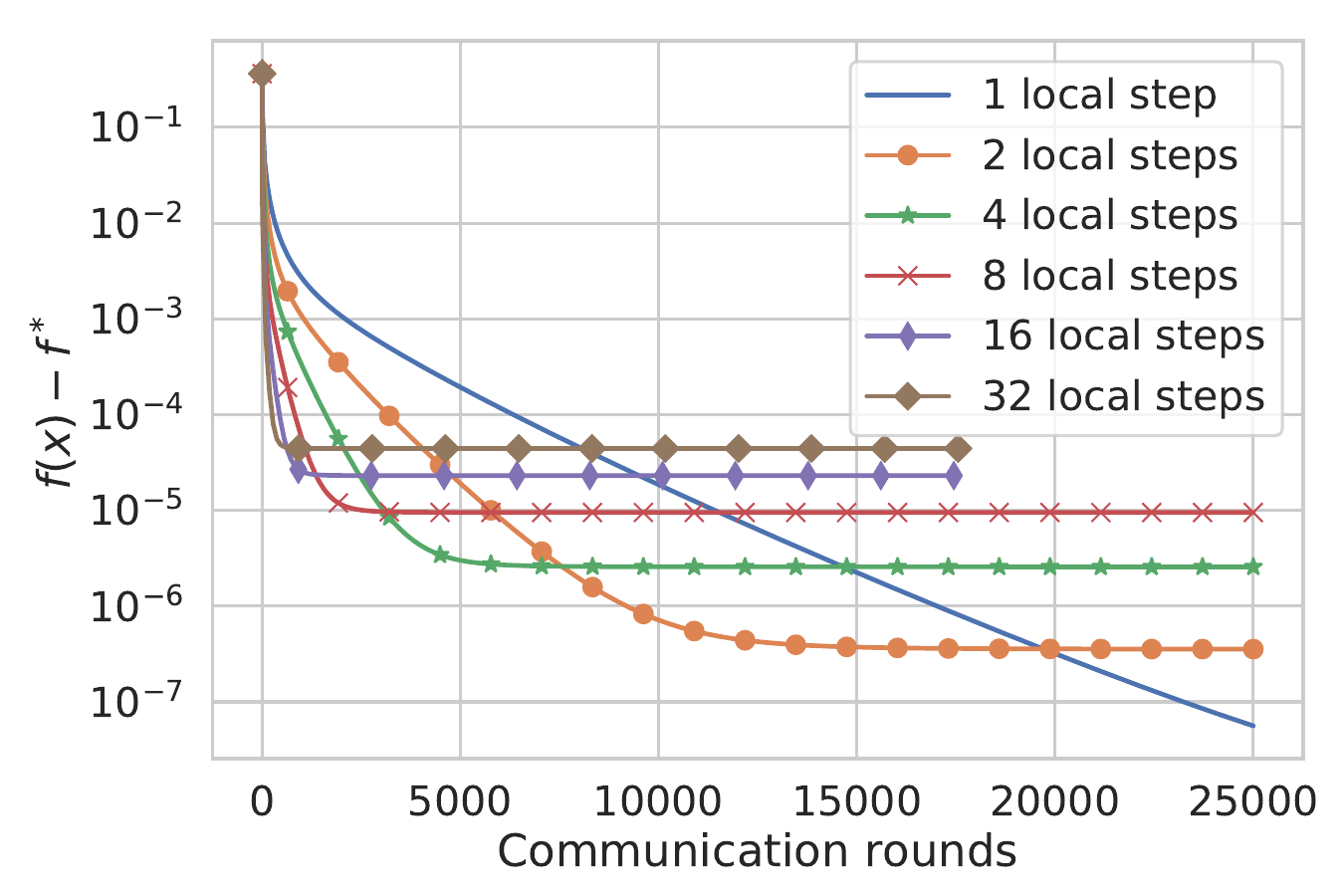}
	\includegraphics[scale=0.23]{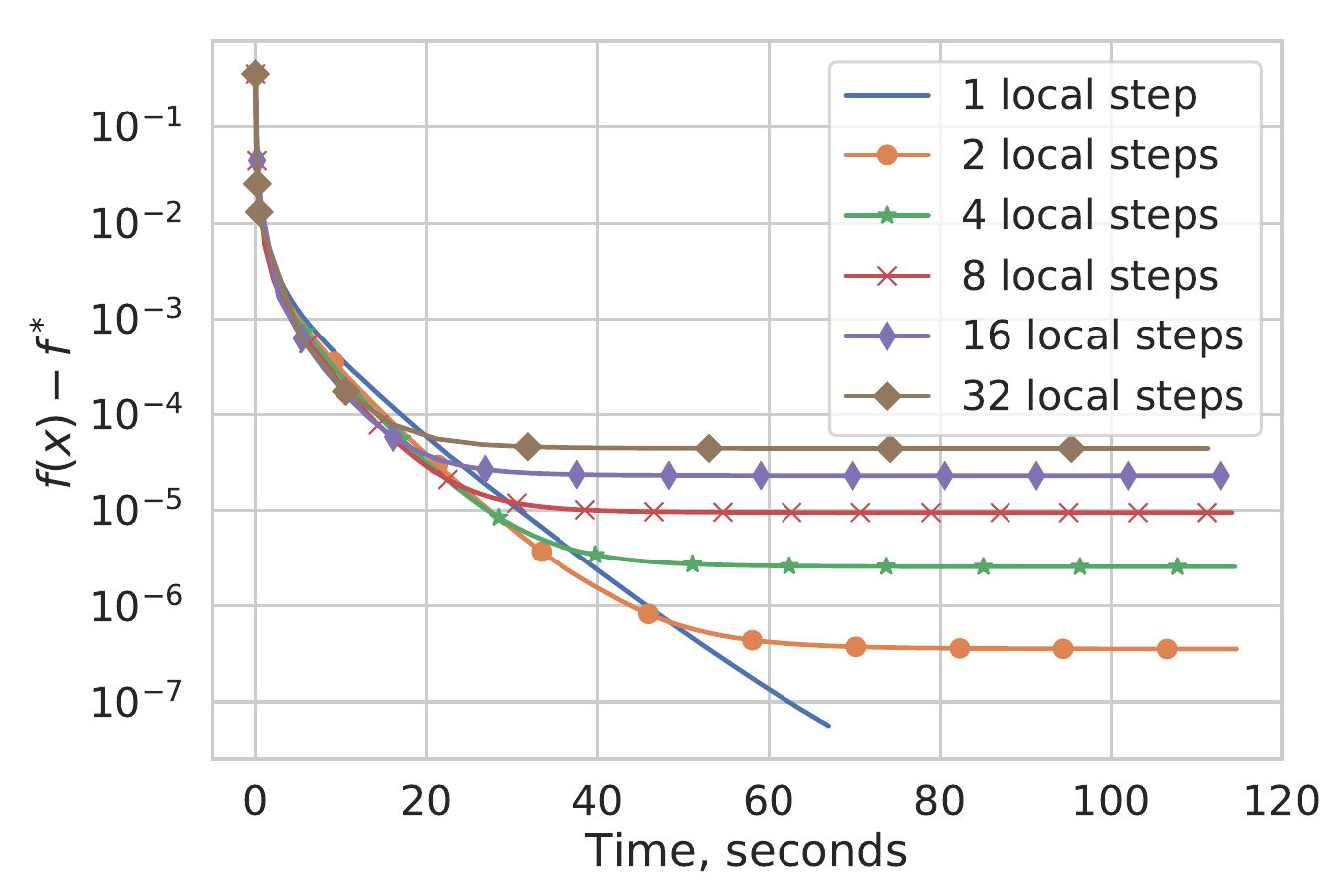}
	\includegraphics[scale=0.23]{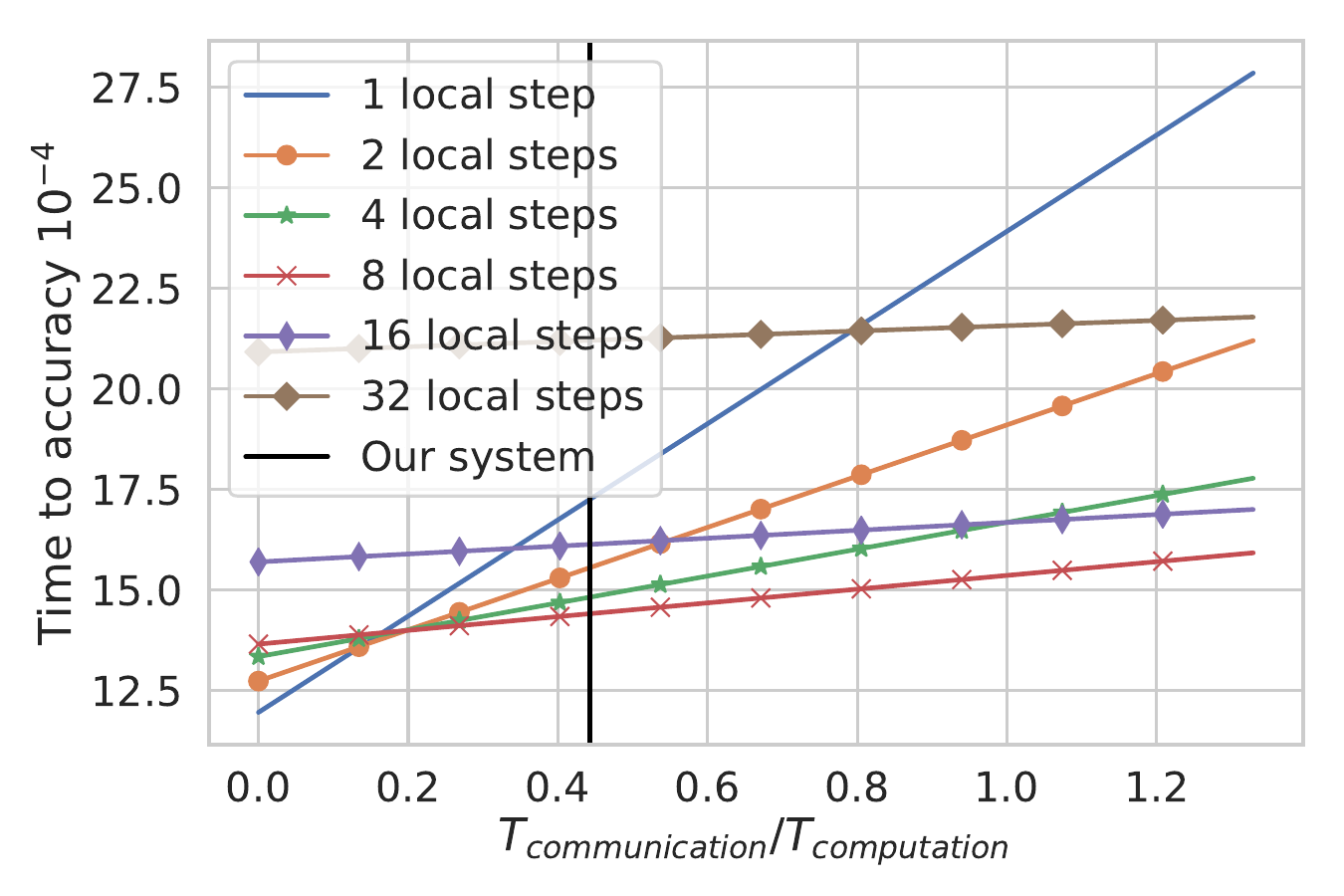}
	\caption{Convergence on heterogeneous data with different number of local steps on the `a5a' dataset. 1 local step corresponds to fully synchronized gradient descent. Left: convergence in terms of communication rounds, which shows a clear advantage of local GD when only limited accuracy is required. Mid plot: wall-clock time might improve only slightly if communication is cheap. Right: what changes with different communication cost.}
	\label{fig:a5a_different_H}
\end{figure*}

\begin{corollary}
    \label{corollary:wc-noniid-unbounded-var}
    Choose $H \leq \frac{\sqrt{T}}{\sqrt{M}}$, then $\gamma = \frac{\sqrt{M}}{8 L \sqrt{T}} \leq \frac{1}{8 H L}$, and hence applying the result of Theorem~\ref{thm:wc-noniid-unbounded-var},
    \begin{align*}
        \begin{split}
            \ec{f (\bar{x}_T) - f (x_\ast)} \leq &\frac{32 L \sqn{x_0 - x_\ast}}{\sqrt{M T}} \\
            &+ \frac{5 \sigmaf^2}{2 L \sqrt{MT}} + \frac{\sigmaf^2 M (H-1)^2}{4 L T}.
        \end{split}
    \end{align*}
\end{corollary}

\textbf{Optimal $H$.} By Corollary~\ref{corollary:wc-noniid-unbounded-var} we see that the optimal value of $H$ is $H = 1 + \floor{T^{1/4} M^{-3/4}}$, which gives $\mathcal{O}\br{\frac{1}{\sqrt{MT}}}$ convergence rate. Thus, the same convergence rate is attained provided that communication is more frequent compared to the identical data regime. 

\section{Experiments}
All experiments described below were run on logistic regression problem with $\ell_2$ regularization of order $\frac{1}{n}$. The datasets were taken from the LIBSVM library~\citep{chang2011libsvm}. The code was written in Python using MPI~\citep{dalcin2011parallel} and run on Intel(R) Xeon(R) Gold 6146 CPU @3.20GHz cores in parallel.
\subsection{Variance measures}
We provide values of $\sigmaf^2$ and $\sigmaopt^2$ in Figure~\ref{fig:variance} for different datasets, minibatch sizes and $M$. The datasets were split evenly without any data reshuffling and no overlaps. For any $M>1$, the value of $\sigmaf$ is lower bounded by $\avemm \|\nabla f_m(x_*)\|^2$ which explains the difference between identical and heterogeneous data. 

\subsection{Identical Data}
For identical data we used $M=20$ nodes and 'a9a' dataset.
We estimated $L$ numerically and ran two experiments, with stepsizes $\frac{1}{L}$ and $\frac{0.05}{L}$ and minibatch size equal 1. In both cases we observe convergence to a neighborhood, although of a different radius. Since we run the experiments on a single machine, the communication is very cheap and there is little gain in time required for convergence. However, the advantage in terms of required communication rounds is self-evident and can lead to significant time improvement under slow communication networks. The results are provided here in Figure~\ref{fig:a5a_same_data_01} and in the supplementary material in Figure~\ref{fig:a5a_same_data_001}.

\subsection{Heterogeneous Data}

\begin{figure}[t]
	\centering
	\includegraphics[scale=0.23]{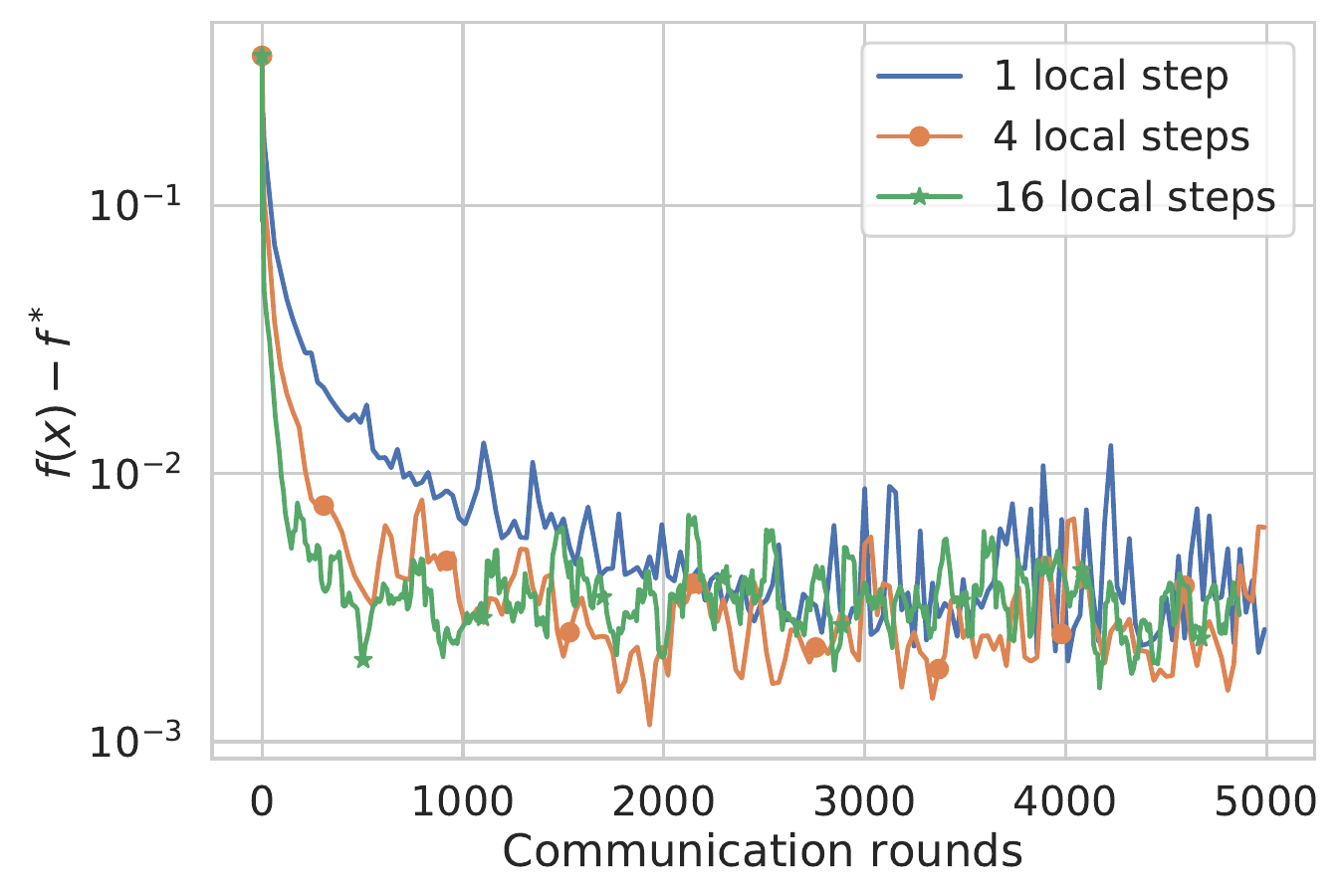}
	\label{fig:a5a_different_H2}
	\caption{Convergence of local SGD on heterogeneous data with different number of local steps on the `a5a' dataset.}
\end{figure}
Since our architecture leads to a very specific trade-off between computation and communication, we  provide plots for the case the communication time relative to gradient computation time is higher or lower. To see the impact of $\sigmaf$, in all experiments we use full gradients $\nabla f_m$ and constant stepsize $\frac{1}{L}$. The data partitioning is not i.i.d.\ and is done based on the index in the original dataset. The results are provided in Figure~\ref{fig:a5a_different_H} and in the supplementary material in Figure~\ref{fig:mushrooms_different_H}.
In cases where communication is significantly more expensive than gradient computation, local methods are much faster for imprecise convergence.    

\section{Acknowledgments}
We thank Li Yipeng for spotting multiple typos in the paper.

\bibliography{local_sgd}

\clearpage
\onecolumn
\part*{Supplementary Material}
\tableofcontents
\clearpage
\section{Basic Facts and Notation}

We use a notation similar to that of~\cite{Stich2018} and denote the sequence of time stamps when synchronization happens as  $(t_{p})_{p=1}^{\infty}$. Given stochastic gradients $g_t^1, g_t^2, \ldots, g_t^M$ at time $t \geq 0$ we define
\begin{align*}
    g_t \eqdef \avemm g_t^m, && \bar{g}_t^m \eqdef \ec{g_t^m} = \begin{cases}
     \nabla f(x_t^m) & \text { for identical data }  \\
     \nabla f_m (x_t^m) & \text { otherwise. } 
\end{cases} && \bar{g}_t \eqdef \ec{g_t}.
\end{align*}

We define an epoch to be a sequence of timesteps between two synchronizations: for $p \in \N$ an epoch is the sequence $t_{t_p}, t_{t_p + 1}, \ldots, t_{t_{p+1} - 1}$.  We summarize some of the notation used in Table~\ref{tab:notation-summary}.
\begin{table}[h]
    \centering
    \caption{Common Notation Summary.}
    \begin{tabular}{|c|c|}
    \hline
    Symbol & Description \\ \hline
    $g_t^m$ & Stochastic gradient at time $t$ on node $m$. See Algorithm~\ref{alg:local_sgd}. \\ \hline
    $x_t^m$ & Local iterate at time $t$ on node $m$. See Algorithm~\ref{alg:local_sgd}. \\ \hline
    $g_t$ & \begin{tabular}[c]{@{}c@{}}Average of stochastic gradients across nodes\\ at time $t$. See Algorithm~\ref{alg:local_sgd}. \end{tabular} \\ \hline
    $\bar{g}_t$ & Expected value of $g_t$: $\ec{g_t} = \bar{g}_t$. \\ \hline
    $\hat{x}_t$ & The average of all local iterates at time $t$. \\ \hline
    $r_{t}$ & The deviation of the average iterate from the optimum $\hat{x}_t - x_\ast$ at time $t$. \\ \hline
    $\sigma^2$ & \begin{tabular}[c]{@{}c@{}}Uniform bound on the variance of the stochastic gradients \\ for identical data. See Assumption~\ref{asm:uniformly-bounded-variance}. \end{tabular}  \\ \hline
    $\sigmaopt^2$ & \begin{tabular}[c]{@{}c@{}}The variance of the stochastic gradients at the optimum \\ for identical data. See Assumption~\ref{asm:finite-sum-stochastic-gradients}. \end{tabular} \\ \hline
    $\sigmaf^2$ & \begin{tabular}[c]{@{}c@{}}The variance of the stochastic gradients at the optimum \\ for heterogeneous data. See Assumption~\ref{asm:finite-sum-stochastic-gradients}. \end{tabular} \\ \hline
    $t_1, t_2, \ldots, t_p$ & Timesteps at which synchronization happen in Algorithm~\ref{alg:local_sgd}. \\ \hline
    $H$ & \begin{tabular}[c]{@{}c@{}}Upper bound on the maximum number of local computations \\ between timesteps, i.e.\ $\max_{p} \abs{t_p - t_{p+1}} \leq H$. \end{tabular}  \\ \hline
    \end{tabular}
    \label{tab:notation-summary}
\end{table}

Throughout the proofs, we will use the variance decomposition that holds for any random vector $X$ with finite second moment:
\begin{align}
    \label{eq:variance_def}
    \ecn {X} = \ecn { X - \ec{X} } + \sqn{\ec{X}}.
\end{align}
In particular, its version for vectors with finite number of values gives
\begin{align}
    \avemm \norm{X_m}^2
    = \avemm \norm{X_m - \aveim X_i}^2 + \norm{\avemm X_m}^2.\label{eq:variance_m}
\end{align}
As a consequence of \eqref{eq:variance_def} we have that,
\begin{align}
    \label{eq:variance_sqnorm_upperbound}
    \ecn{X - \ec{X}} \leq \ecn{X}.
\end{align}

\begin{proposition}[Jensen's inequality]\label{pr:jensen}
    For any convex function $f$ and any vectors $x^1,\dotsc, x^M$ we have
    \begin{align}
        f\br{\avemm x^m}
        \le \avemm f(x^m). \label{eq:jensen}
    \end{align}
    As a special case with $f(x)=\|x\|^2$, we obtain
    \begin{align}
        \norm{\avemm x_m}^2
        \le \avemm \|x_m\|^2.\label{eq:jensen_norm}
    \end{align}
 \end{proposition}
\noindent We denote the Bregman divergence associated with function $f$ and arbitrary $x, y$ as
\begin{align*}
D_f(x, y)
\eqdef f(x) - f(y) - \ev{\nabla f(y), x - y}.
\end{align*}
\begin{proposition}
     If $f$ is $L$-smooth and convex, then for any $x$ and $y$ it holds
     \begin{align}
          \|\nabla f(x) - \nabla f(y)\|^2
          \le 2LD_f(x, y). \label{eq:bregman-dif}
     \end{align}
\end{proposition}
\noindent If $f$ satisfies Assumption~\ref{asm:convexity-and-smoothness}, then
\begin{equation}
     \label{eq:asm-strong-convexity}
     f(x) + \ev{\nabla f(y), x - y} + \frac{\mu}{2} \sqn{y - x} \leq f(y) \qquad \forall x, y \in \R^d.
\end{equation}
\noindent We will also use the following facts from linear algebra:
\begin{align}
    \sqn{x + y} &\leq 2 \sqn{x} + 2 \sqn{y}, \label{eq:sum_sqnorm} \\
    2 \ev{a, b} &\leq \zeta \sqn{a} + \zeta^{-1} \sqn{b} \text { for all } a, b \in \R^d \text { and } \zeta > 0. \label{eq:youngs-inequality}
\end{align}

\section{Proofs for Identical data under Assumption~\ref{asm:uniformly-bounded-variance}}
\subsection{Proof of Lemma~\ref{lemma:uniform-var-iterate-variance-bound}}
\begin{proof}
     Let $t \in \N$ be such that $t_p \leq t \leq t_{p+1} - 1$. Recall that for a time $t$ such that $t_p \leq t < t_{p+1}$ we have $x_{t+1}^m = x_t^m - \gamma g_{t}^m$ and $\hat{x}_{t+1} = \hat{x}_t - \gamma g_t$. Hence for the expectation conditional on $x_t^1, x_t^2, \ldots, x_t^M$ we have:
     \begin{align*}
         \ecn{x_{t+1}^m - \hat{x}_{t+1}} &= \sqn{x_t^m - \hat{x}_t} + \gamma^2 \ecn{g_t^m - g_t} - 2 \gamma \ec{\ev{x_t^m - \hat{x}_t, g_t^m - g_t}} \\
         &= \sqn{x_t^m - \hat{x}_t} + \gamma^2 \ecn{g_t^m - g_t} - 2 \gamma \ev{x_t^m - \hat{x}_t, \nabla f(x_{t}^m)} \\
         &\quad + 2 \gamma \ev{x_t^m - \hat{x}_t, \bar{g}_t}.
     \end{align*}
     Averaging both sides and letting $V_t = \frac{1}{M} \sum_{m} \sqn{x_{t}^m - \hat{x}_t}$, we have
     \begin{align}
         \ec{V_{t+1}} &= V_t + \frac{\gamma^2}{M} \sum_{m} \ecn{g_t^m - g_t} - \frac{2 \gamma}{M} \sum_{m} \ev{x_t^m - \hat{x}_t, \nabla f(x_t^m)} + 2 \gamma \underbrace{\ev{ \hat{x}_t - \hat{x}_t, \bar{g}_t}}_{=0} \nonumber \\
         \label{iterate-variance-bound-recursion}
         &= V_t + \frac{\gamma^2}{M} \sum_{m} \ecn{g_t^m - g_t} - \frac{2 \gamma}{M} \sum_{m} \ev{x_t^m - \hat{x}_t, \nabla f(x_t^m)}.
     \end{align}
     Now note that by expanding the square we have,
     \begin{align}
         \label{iterate-gradient-variance-bound-1}
         \ecn{g_t^m - g_t} &= \ecn{g_t^m - \bar{g_t}} + \ecn{\bar{g}_t - g_t} + 2 \ec{\ev{g_t^m - \bar{g}_t, \bar{g}_t - g_t}}.
     \end{align}
     We decompose the first term in the last equality again by expanding the square,
     \begin{align*}
         \ecn{g_t^m - \bar{g}_t} &= \ecn{g_t^m - \bar{g}_t^m} + \sqn{\bar{g}_t^m - \bar{g}_t} + 2 \ec{\ev{ g_t^m - \bar{g}_t^m, \bar{g}_t^m - \bar{g}_t }} \\
         &= \ecn{g_t^m - \bar{g}_t^m} + \sqn{\bar{g}_t^m - \bar{g}_t} + 2 \underbrace{\ev{ \bar{g}_t^m - \bar{g}_t^m, \bar{g}_t^m - \bar{g}_t }}_{=0} \\
         &= \ecn{g_t^m - \bar{g}_t^m} + \sqn{\bar{g}_t^m - \bar{g}_t}.
     \end{align*}
     Plugging this into~\eqref{iterate-gradient-variance-bound-1} we have,
     \begin{align*}
         \ecn{g_t^m - g_t} &= \ecn{g_t^m - \bar{g}_t^m} + \sqn{\bar{g}_t^m - \bar{g}_t} + \ecn{\bar{g}_t - g_t} + 2 \ec{\ev{g_t^m - \bar{g}_t, \bar{g}_t - g_t}}.
     \end{align*}
     Now average over $m$:
     \begin{align*}
         \frac{1}{M} \sum_{m} \ecn{g_t^m - g_t} &= \frac{1}{M} \sum_{m} \ecn{g_t^m - \bar{g}_t^m} + \frac{1}{M} \sum_m \sqn{\bar{g}_t^m - \bar{g}_t} + \ecn{\bar{g}_t - g_t} \\
         &- 2 \ecn{\bar{g}_t - g_t},
     \end{align*}
     where we used that by definition $\avemm g_t^m = g_t$. Hence,
     \begin{align}
         \frac{1}{M} \sum_m \ecn{g_t^m - g_t} &= \frac{1}{M} \sum_{m} \ecn{g_t^m - \bar{g}_t^m} + \frac{1}{M} \sum_m \sqn{\bar{g}_t^m - \bar{g}_t} - \ecn{\bar{g}_t - g_t} \nonumber \\
         \label{iterate-gradient-variance-bound-2}
         &\leq \frac{1}{M} \sum_{m} \ecn{g_t^m - \bar{g}_t^m} + \frac{1}{M} \sum_{m} \sqn{\bar{g}_t^m - \bar{g}_t}.
     \end{align}
     Now note that for the first term in~\eqref{iterate-gradient-variance-bound-2} we have by Assumption~\ref{asm:uniformly-bounded-variance},
     \begin{align}
         \label{iterate-gradient-variance-bound-3}
         \ecn{g_t^m - \bar{g}_t^m} &= \ecn{g_t^m - \nabla f(x_t^m)} \leq \sigma^2.
     \end{align}
     For the second term in \eqref{iterate-gradient-variance-bound-2} we have
     \begin{align*}
         \sqn{\bar{g}_t^m - \bar{g}_t} &= \sqn{\bar{g}_t^m - \nabla f(\hat{x}_t)} + \sqn{ \nabla f(\hat{x}_t) - \bar{g}_t } + 2 \ev{ \bar{g}_t^m - \nabla f(\hat{x}_t), \nabla f(\hat{x}_t) - \bar{g}_t  }.
     \end{align*}
     Averaging over $m$,
     \begin{align*}
         \frac{1}{M} \sum_{m=1}^{M} \sqn{ \bar{g}_t^m - \bar{g}_t } &= \frac{1}{M} \sum_{m} \sqn{\bar{g}_t^m - \nabla f(\hat{x}_t)} + \sqn{\nabla f(\hat{x}_t) - \bar{g}_t} + 2 \ev{ \bar{g}_t - \nabla f(\hat{x}_t), \nabla f(\hat{x}_t) - \bar{g}_t } \\
         &= \frac{1}{M} \sum_{m} \sqn{\bar{g}_t^m - \nabla f(\hat{x}_t)} + \sqn{\nabla f(\hat{x}_t) - \bar{g}_t} - 2 \sqn{\nabla f(\hat{x}_t) - \bar{g}_t} \\
         &= \frac{1}{M} \sum_{m} \sqn{\bar{g}_t^m - \nabla f(\hat{x}_t)} - \sqn{\nabla f(\hat{x}_t) - \bar{g}_t} \leq \frac{1}{M} \sum_{m} \sqn{\bar{g}_t^m - \nabla f(\hat{x}_t)},
     \end{align*}
     where we used the fact that $\frac{1}{M} \sum_{m} \bar{g}_t^m = \bar{g}_t$, which comes from the linearity of expectation. Now we bound $\sqn{\bar{g}_t^m - \nabla f(\hat{x}_t)}$ in the last inequality by smoothness and then use that Jensen's inequality implies $\sum_{m=1}^{M} (f(\hat{x}_t) - f(x_t^m)) \leq 0$,
     \begin{align}
         \frac{1}{M} \sum_{m} \sqn{ \bar{g}_t^m - \nabla f(\hat{x}_t) } &= \frac{1}{M} \sum_{m} \sqn{ \nabla f(x_{t}^m) - \nabla f(\hat{x}_t) } \nonumber \\
         &\overset{\eqref{eq:bregman-dif}}{\le} \frac{1}{M} \sum_{m} 2L (f(\hat x_t) - f(x_t^m) - \ev{\hat x_t - x_t^m, \nabla f(x_t^m)}) \nonumber \\
         \label{eq:iterate-gradient-variance-bound-4}
         &\le \frac{2L}{M} \sum_{m} \ev{x_t^m - \hat x_t, \nabla f(x_t^m)}.
     \end{align}
     Plugging in \eqref{eq:iterate-gradient-variance-bound-4} and \eqref{iterate-gradient-variance-bound-3} into \eqref{iterate-gradient-variance-bound-2} we have,
     \begin{align}
         \label{iterate-gradient-variance-bound-5}
         \frac{1}{M} \sum_{m} \ecn{g_t^m - g_t} \leq \sigma^2 + \frac{2L}{M} \sum_{m} \ev{x_t^m - \hat x_t, \nabla f(x_t^m)}.
     \end{align}
     Plugging \eqref{iterate-gradient-variance-bound-5} into \eqref{iterate-variance-bound-recursion}, we get
     \begin{align}
         \ec{V_{t+1}} 
         &\le V_t + \gamma^2 \sigma^2 - \frac{2 \gamma}{M} \sum_m \ev{x_t^m - \hat{x}_t, \nabla f(x_t^m)} + \frac{2L \gamma^2}{M} \sum_{m} \ev{x_t^m - \hat x_t, \nabla f(x_t^m)} \nonumber \\
         \label{iterate-variance-recursion-pre-sc}
         &= V_t + \gamma^2 \sigma^2 - \frac{2 \gamma(1 - \gamma L)}{M} \sum_m \ev{x_t^m - \hat{x}_t, \nabla f(x_t^m)} \\
         &\overset{\eqref{eq:asm-strong-convexity}}{\le} \br{1 - \gamma(1 - \gamma L)\mu} V_t + \gamma^2\sigma^2. \nonumber
     \end{align}
     Using that $\gamma \leq \frac{1}{2L}$ we can conclude,
     \begin{align*}
         \ec{V_{t+1}}
         &\le \br{1 - \frac{\gamma\mu}{2}} V_t + \gamma^2\sigma^2 \\
         &\le V_t + \gamma^2 \sigma^2.
     \end{align*}
     Taking expectations and iterating the above inequality,
     \begin{align*}
         \ec{V_t} &\leq \ec{V_{t_p}} + \gamma^2 \sigma^2 \br{t - t_p} \\
         &\leq \ec{V_{t_p}} + \gamma^2 \sigma^2 \br{t_{p+1} - t_p - 1} \\
         &\leq \ec{V_{t_p}} + \gamma^2 \sigma^2 \br{H - 1}.
     \end{align*}
     It remains to notice that by assumption we have $V_{t_p} = 0$.
\end{proof}

\subsection{Two More Lemmas}
\begin{lemma}{\citep{Stich2018}.}
     \label{lemma:iterate-one-recursion}
     Let $(x_t^m)_{t \geq 0}$ be iterates generated by Algorithm~\ref{alg:local_sgd} run with identical data. Suppose that $f$ satisfies Assumption~\ref{asm:convexity-and-smoothness} and that $\gamma \leq \frac{1}{2 L}$. Then,
     \begin{align}
         \begin{split}
             \ecn{\hat{x}_{t+1} - x_\ast} \leq (1 &- \gamma \mu) \ecn{\hat{x}_t - x_\ast} + \gamma^2 \ecn{g_t - \bar{g}_t} \\
             &- \frac{\gamma}{2} \ec{D_{f} (\hat{x}_t, x_\ast)} + 2 \gamma L \ec{V_t}.
         \end{split}
     \end{align}
 \end{lemma}
 \begin{proof}
     This is Lemma 3.1 in \citep{Stich2018}.
 \end{proof}
 
 \begin{lemma}
     \label{lemma:minibatch-variance-reduction}
     Suppose that Assumption~\ref{asm:uniformly-bounded-variance} holds. Then,
     \[ \ecn{g_t - \bar{g}_t} \leq \frac{\sigma^2}{M}. \]
 \end{lemma}
 \begin{proof}
     This is Lemma 3.2 in \citep{Stich2018}. Because the stochastic gradients $g_t^m$ are independent we have that the variance of their sum is the sum of their variances, hence
     \begin{align*}
         \ecn{g_t - \bar{g}_t} &= \frac{1}{M^2} \ecn{ \sum_{m=1}^{M} g_t^m - \bar{g}_t^m } = \frac{1}{M^2} \sum_{m=1}^{M} \ecn{g_t^m - \bar{g}_t^m} \leq \frac{\sigma^2}{M}.
     \end{align*}
 \end{proof}

\subsection{Proof of Theorem~\ref{thm:sc-convergence-theorem}}
\begin{proof}
     Combining Lemma~\ref{lemma:iterate-one-recursion} and Lemma~\ref{lemma:minibatch-variance-reduction}, we have
     \begin{equation}
         \label{sc-thm-proof-1}
         \ecn{\hat{x}_{t+1} - x_\ast} \leq (1 - \gamma \mu) \ecn{\hat{x}_t - x_\ast} + \frac{\gamma^2 \sigma^2}{M} - \frac{\gamma}{2} \ec{D_{f} (\hat{x}_t, x_\ast)} + 2 \gamma L \ec{V_t}.
     \end{equation}
     Using Lemma~\ref{lemma:uniform-var-iterate-variance-bound} we can upper bound the $\ec{V_t}$ term in $\eqref{sc-thm-proof-1}$:
     \begin{equation*}
         \ecn{\hat{x}_{t+1} - x_\ast} \leq (1 - \gamma \mu) \ecn{\hat{x}_t - x_\ast} + \frac{\gamma^2 \sigma^2}{M} - \frac{\gamma}{2} \ec{D_{f} (\hat{x}_t, x_\ast)} + 2 \gamma^3 L \br{H - 1} \sigma^2.
     \end{equation*}
     Letting $r_{t+1} = \hat{x}_{t+1} - x_\ast$ and we have,
     \begin{equation*}
         \ecn{r_{t+1}} \leq \br{1 - \gamma \mu} \ecn{r_t} + \frac{\gamma^2 \sigma^2}{M} + 2 \gamma^3 L \br{H - 1} \sigma^2.
     \end{equation*}
     Recursing the above inequality we have,
     \begin{equation*}
         \ecn{r_{T}} \leq \br{1 - \gamma \mu}^{T} \ecn{r_0} + \br{ \sum_{t=0}^{T-1} \br{1 - \gamma \mu}^t } \br{ \frac{\gamma^2 \sigma^2}{M} + 2 \gamma^3 L \br{H - 1} \sigma^2 }.
     \end{equation*}
     Using that $\sum_{t=0}^{T-1} \br{1 - \gamma \mu}^{t} \leq \sum_{t=0}^{\infty} \br{1 - \gamma \mu}^{t} = \frac{1}{\gamma \mu}$ we have,
     \begin{equation*}
         \ecn{r_{T}} \leq \br{1 - \gamma \mu}^{T} \ecn{r_0} + \frac{\gamma \sigma^2}{\mu M} + \frac{2 \gamma^2 L \br{H - 1} \sigma^2}{\mu},
     \end{equation*}
     which is the claim of this theorem.
 \end{proof}

\subsection{Proof of Theorem~\ref{thm:weakly-convex-thm}}
\begin{proof}
    Let $r_t = \hat{x}_t - x_\ast$, then putting $\mu = 0$ in Lemma~\ref{lemma:iterate-one-recursion} and combining it with Lemma~\ref{lemma:minibatch-variance-reduction}, we have
    \begin{equation*}
        \ecn{r_{t+1}} \leq \ecn{r_t} + \frac{\gamma^2 \sigma^2}{M} - \frac{\gamma}{2} \ec{D_{f} (\hat{x}_t, x_\ast)} + 2 \gamma L \ec{V_t}.
    \end{equation*}
    Further using Lemma~\ref{lemma:uniform-var-iterate-variance-bound},
    \begin{equation*}
        \ecn{r_{t+1}} \leq \ecn{r_t} + \frac{\gamma^2 \sigma^2}{M} - \frac{\gamma}{2} \ec{D_{f} (\hat{x}_t, x_\ast)} + 2 \gamma^3 L \br{H - 1} \sigma^2.
    \end{equation*}
    Rearranging we have,
    \begin{align*}
        \frac{\gamma}{2} \ec{D_{f} (\hat{x}_t, x_\ast)} \leq \ecn{r_t} - \ecn{r_{t+1}} + \frac{\gamma^2 \sigma^2}{M} + 2 \gamma^3 L \br{H - 1} \sigma^2.
    \end{align*}
    Averaging the above equation as $t$ varies between $0$ and $T-1$,
    \begin{align}
        \frac{\gamma}{2 T} \sum_{t=0}^{T-1} \ec{D_{f} (\hat{x}_t, x_\ast)} &\leq \frac{1}{T} \sum_{t=0}^{T-1} \ecn{r_t} - \ecn{r_{t+1}} + \frac{1}{T} \sum_{t=0}^{T-1} \br{ \frac{\gamma^2 \sigma^2}{M} + 2 \gamma^3 L \br{H - 1} \sigma^2} \nonumber \\
        &= \frac{\sqn{r_0} - \ecn{r_T}}{T} + \frac{\gamma^2 \sigma^2}{M} + 2 \gamma^3 L \br{H - 1} \sigma^2 \nonumber \\
        \label{wc-thm-proof-2}
        &\leq \frac{\sqn{r_0}}{T} + \frac{\gamma^2 \sigma^2}{M} + 2 \gamma^3 L \br{H - 1} \sigma^2.
    \end{align}
    By Jensen's inequality we have $D_{f} (\bar{x}_T, x_\ast) \leq \frac{1}{T} \sum_{t=0}^{T-1} D_{f} (\hat{x}_t, x_\ast)$. Using this in \eqref{wc-thm-proof-2} we have,
    \begin{align*}
        \frac{\gamma}{2} \ec{D_{f} (\hat{x}_t, x_\ast)} \leq \frac{\sqn{r_0}}{T} + \frac{\gamma^2 \sigma^2}{M} + 2 \gamma^3 L \br{H - 1} \sigma^2.
    \end{align*}
    Dividing both sides by $\gamma/2$ yields the theorem's claim.
\end{proof}

\section{Proofs for identical data under Assumption~\ref{asm:finite-sum-stochastic-gradients}}
\subsection{Preliminary Lemmas}
\begin{lemma}
     \label{lemma:gradient-smoothness-bound}
     Individual gradient variance bound: assume that Assumption~\ref{asm:finite-sum-stochastic-gradients} holds with identical data, then for all $t \geq 0$ and  $m \in [M]$ we have
     \begin{equation}
         \label{eq:lma-gradient-smoothness-bound}
         \ecn{g_t^m} \leq 4 L D_{f} (x_t^m, x_\ast) + 2 \sigma_{m}^2,
     \end{equation}
    where $\sigma_m^2 \eqdef \ecn[z_m \sim \D_m]{\nabla f (x_\ast, z_m)}$ is the noise at the optimum on the $m$-th node.
\end{lemma}
\begin{proof}
     Using that $g_t^m = \nabla f (x_t^m, z_m)$ for some $z_m \sim \D_m$, 
     \begin{align*}
         \sqn{g_t^m} &= \sqn{\nabla f (x_t^m, z_m)} \\
         &\overset{\eqref{eq:sum_sqnorm}}{\leq} 2 \sqn{\nabla f (x_t^m, z_m) - \nabla f (x_\ast, z_m)} + 2 \sqn{\nabla f (x_\ast, z_m)} \\
         &\overset{\eqref{eq:bregman-dif}}{\leq} 4 L \br{f (x_t^m, z_m) - f (x_\ast, z_m) - \ev{\nabla f (x_\ast, z_m), x_t^m - x_\ast}} + 2 \sqn{\nabla f (x_\ast, z_m)}.
     \end{align*}
     Taking expectations and using that $\ec{\nabla f (x_\ast, z)} = \nabla f(x_\ast) = 0$ we get,
     \begin{align*}
         \ecn{g_t^m} &\leq 4 L \br{f(x_t^m) - f(x_\ast)} + 2 \sigma_m^2 \\
         &= 4 L D_{f} (x_t^m, x_\ast) + 2 \sigma_m^2.
     \end{align*}
\end{proof}

\begin{lemma}[Average gradient variance reduction]
    \label{lemma:minibatch-gradient-noise-reduction}
    Assume that Assumption~\ref{asm:finite-sum-stochastic-gradients} holds with identical data, then for all $t \geq 0$ and for $M$ nodes we have,
    \begin{align}
        \label{eq:lma-mbns}
        \ecn{g_t - \bar{g}_t} \leq \frac{2 \sigmaopt^2}{M} + \frac{4 L}{M^2} \sum_{m=1}^{M} D_{f} (x_t^m, x_\ast).
    \end{align}
\end{lemma}
\begin{proof}
    Using the definition of $g_t$ and $\bar{g}_t$,
    \begin{align}
        \ecn{g_t - \bar{g}_t} &= \ecn{\frac{1}{M} \sum_{m=1}^{M} g_t^m - \nabla f(x_t^m)} \nonumber \\
        &= \frac{1}{M^2} \ecn{\sum_{m=1}^{M} \br{g_t^m - \nabla f(x_t^m)}}. \label{eq:lma-mgns-proof-1}
    \end{align}
    The sum in \eqref{eq:lma-mgns-proof-1} is the variance of a sum of independent random variables and hence can be decomposed into the sum of their individual variances which we can use Lemma~\ref{lemma:gradient-smoothness-bound} to bound:
    \begin{align*}
        \ecn{g_t - \bar{g}_t} &= \frac{1}{M^2} \sum_{m=1}^{M} \ecn{g_t^m - \nabla f(x_t^m)} \\
        &\overset{\eqref{eq:variance_sqnorm_upperbound}}{\leq} \frac{1}{M^2} \sum_{m=1}^{M} \ecn{g_t^m} \\
        &\overset{\eqref{eq:lma-gradient-smoothness-bound}}{\leq} \frac{1}{M^2} \sum_{m=1}^{M} \br{ 2 \sigma_m^2 + 4 L D_{f} (x_t^m, x_\ast) } \\
        &= \frac{2 \sigmaopt^2}{M} + \frac{4 L}{M^2} \sum_{m=1}^{M} D_f (x_t^m, x_\ast),
    \end{align*}
    where in the last equality we used that $\sigmaopt^2$ is by definition equal to $\sum_{m=1}^{M} \sigma_m^2/M$.
\end{proof}

\begin{lemma}
    \label{lemma:perturbed-iterate-analysis}
    Perturbed iterate analysis: this bounds the optimality gap across one iteration when the descent step is $\hat{x}_{t+1} = \hat{x}_t - \frac{\gamma}{M} \sum_{m=1}^{M} \nabla f(x_t^m)$, i.e.\ when the expectation of the local SGD update is used. Suppose that Assumptions~\ref{asm:convexity-and-smoothness}, and \ref{asm:finite-sum-stochastic-gradients} hold with identical data. Then,
    \begin{align}
        \begin{split}
            \label{eq:lma-perturbed-iterate-analysis}
            \sqn{\hat{x}_{t+1} - \gamma \bar{g}_t - x_\ast} \leq \sqn{\hat{x}_t - x_\ast} &+ 2 \gamma L V_t \\
            &+ \frac{2 \gamma}{M} \sum_{m=1}^{M} \br{\br{\gamma L - \frac{1}{2}} \br{f(x_t^m) - f(x_\ast)} - \frac{\mu}{2} \sqn{x_t^m - x_\ast}}.
        \end{split}
    \end{align}
\end{lemma}
\begin{proof}
    This is the first part of Lemma~3.1 in~\citep{Stich2018} and we reproduce it for completeness:
    \begin{align}
        \sqn{\hat{x}_t - x_\ast - \gamma \bar{g}_t} = \sqn{\hat{x}_t - x_\ast} &+ \gamma^2 \sqn{\bar{g}_t} - 2 \gamma \ev{\hat{x}_t - x_\ast, \bar{g}_t} \nonumber \\
        = \sqn{\hat{x}_t - x_\ast} &+ \gamma^2 \sqn{\bar{g}_t} - \frac{2 \gamma}{M} \sum_{m=1}^{M} \ev{\hat{x}_t - x_\ast, \nabla f(x_t^m)} \nonumber \\
        \overset{\eqref{eq:jensen_norm}}{\leq} \sqn{\hat{x}_t - x_\ast} &+ \frac{\gamma^2}{M} \sum_{m=1}^{M} \sqn{\nabla f(x_t^m)} - \frac{2 \gamma}{M} \sum_{m=1}^{M} \ev{\hat{x}_t - x_t^m + x_t^m - x_\ast, \nabla f(x_t^m)} \nonumber \\
        \begin{split}
            = \sqn{\hat{x}_t - x_\ast} &+ \frac{\gamma^2}{M} \sum_{m=1}^{M} \sqn{\nabla f(x_t^m) - f(x_\ast)} - \frac{2 \gamma}{M} \sum_{m=1}^{M} \ev{x_t^m - x_\ast, \nabla f(x_t^m)} \\
            &- \frac{2 \gamma}{M} \sum_{m=1}^{M} \ev{\hat{x}_t - x_t^m, \nabla f(x_t^m)}
        \end{split} \nonumber \\
        \begin{split}
            \overset{\eqref{eq:bregman-dif}}{\leq} \sqn{\hat{x}_t - x_\ast} &+ \frac{2 L \gamma^2}{M} \sum_{m=1}^{M} \br{f(x_t^m) - f(x_\ast)} - \frac{2 \gamma}{M} \sum_{m=1}^{M} \ev{x_t^m - x_\ast, \nabla f(x_t^m)} \\
            &- \frac{2 \gamma}{M} \sum_{m=1}^{M} \ev{\hat{x}_t - x_t^m, \nabla f(x_t^m)}
        \end{split} \nonumber \\
        \begin{split}\label{eq:lma-pia-proof-1}
            \overset{\eqref{eq:asm-strong-convexity}}{\leq} \sqn{\hat{x}_t - x_\ast} &+ \frac{2 \gamma}{M} \sum_{m=1}^{M} \br{\br{ \gamma L - 1 } \br{f(\hat{x}_t^m) - f(x_\ast)} - \frac{\mu}{2} \sqn{x_t^m - x_\ast}} \\
            &- \frac{2 \gamma}{M} \sum_{m=1}^{M} \ev{\hat{x}_t - x_t^m, \nabla f(x_t^m)}.
        \end{split}
    \end{align}
    To bound the last term in (\ref{eq:lma-pia-proof-1}) we use the generalized Young's inequality $2 \ev{a, b} \leq \zeta \sqn{a} + \zeta^{-1} \sqn{b}$ with $\zeta = 2L$:
    \begin{align}
        - 2 \ev{\hat{x}_t - x_t^m, \nabla f(x_t^m)} &\overset{\eqref{eq:youngs-inequality}}{\leq} 2L \sqn{x_t^m - \hat{x}_t} + \frac{1}{2L} \sqn{\nabla f(x_t^m)} \nonumber \\
        &= 2L \sqn{x_t^m - \hat{x}_t} + \frac{1}{2L} \sqn{\nabla f(x_t^m) - f(x_\ast)} \nonumber \\
        &\overset{\eqref{eq:bregman-dif}}{\leq} 2L \sqn{x_t^m - \hat{x}_t} + \br{f(x_t^m) - f(x_\ast)}. \label{eq:lma-pia-proof-2}
    \end{align}
    Finally, using \eqref{eq:lma-pia-proof-2} in \eqref{eq:lma-pia-proof-1} we get,
    \begin{align*}
        \begin{split}
            \sqn{\hat{x}_t - \gamma \bar{g}_t - x_\ast} \overset{\eqref{eq:lma-pia-proof-1}, \eqref{eq:lma-pia-proof-2}}{\leq} \sqn{\hat{x}_t - x_\ast} &+ \frac{2 \gamma}{M} \sum_{m=1}^{M} \br{\br{ \gamma L - \frac{1}{2} } \br{f(\hat{x}_t^m) - f(x_\ast)} - \frac{\mu}{2} \sqn{x_t^m - x_\ast}} \\
            &+ \frac{2 \gamma L}{M} \sum_{m=1}^{M} \sqn{\hat{x}_t - x_t^m}.
        \end{split}
    \end{align*}
\end{proof}

\begin{lemma}
    \label{lemma:optimality-gap-contraction}
    Single-iterate optimality gap analysis: Suppose that Assumptions~\ref{asm:convexity-and-smoothness} and \ref{asm:finite-sum-stochastic-gradients} hold with identical data. Choose a stepsize $\gamma > 0$ such that $\gamma \leq \frac{1}{4 L \br{1 + \frac{2}{M}}}$ where $M$ is the number of nodes, then for expectation conditional on $x_t^1, x_t^2, \ldots, x_t^M$ we have
    \begin{align}
        \label{eq:lma-optimality-gap-contraction}
        \ecn{\hat{x}_{t+1} - x_\ast} \leq \br{1 - \gamma \mu} \sqn{\hat{x}_t - x_\ast} &+ 2 \gamma L V_t + \frac{2 \gamma^2 \sigmaopt^2}{M} - \frac{\gamma}{2} \br{f(\hat{x}_t) - f(x_\ast)},
    \end{align}
    where $\hat{x}_t = \frac{1}{M} \sum_{m=1}^{M} x_t^m$ and $V_t \eqdef \frac{1}{M} \sum_{m=1}^{M} \sqn{x_t^m - \hat{x}_t}$ is the iterate variance across the different nodes from their mean at timestep $t$.
\end{lemma}
\begin{proof}
    This is a modification of Lemma~3.1 in~\cite{Stich2018}. For expectation conditional on $(x_t^m)_{m=1}^{M}$ and using Lemma~\ref{lemma:perturbed-iterate-analysis},
    \begin{align*}
        \ecn{\hat{x}_{t+1} - x_\ast} &\overset{\eqref{eq:variance_def}}{=} \sqn{\hat{x}_t - x_\ast - \gamma \bar{g}_t} + \gamma^2 \ecn{g_t - \bar{g}_t} \\
        &\overset{\eqref{eq:lma-perturbed-iterate-analysis}}{\leq} \sqn{\hat{x}_t - x_\ast} + 2 \gamma L V_t + \gamma^2 \ecn{g_t - \bar{g}_t}\\
        &+ \frac{2 \gamma}{M} \sum_{m=1}^{M} \br{ \br{\gamma L - \frac{1}{2}} \br{f(x_t^m) - f(x_\ast)} - \frac{\mu}{2} \sqn{x_t^m - x_\ast} }.
    \end{align*}
    Now use Lemma~\ref{lemma:minibatch-gradient-noise-reduction} to bound $\sqn{g_t - \bar{g}_t}$:
    \begin{align}
        \label{eq:lma-ogc-proof-1}
        \begin{split}
            \ecn{\hat{x}_{t+1} - x_\ast} \overset{\eqref{eq:lma-mbns}}{\leq} \sqn{\hat{x}_t - x_\ast} &+ 2 \gamma L V_t + \frac{2 \gamma^2 \sigmaopt^2}{M} \\
            &+ \frac{2 \gamma}{M} \sum_{m=1}^{M} \br{ \br{\gamma L + \frac{2 \gamma L}{M} - \frac{1}{2}} \br{f(x_t^m) - f(x_\ast)} - \frac{\mu}{2} \sqn{x_t^m - x_\ast} }.
        \end{split}
    \end{align}
    We now use that the stepsize $\gamma \leq \frac{1}{4 L \br{1 + \frac{2}{M}}}$ to bound the last term in \eqref{eq:lma-ogc-proof-1},
    \begin{align}
        \label{eq:lma-ogc-proof-2}
        \begin{split}
            \ecn{\hat{x}_{t+1} - x_\ast} \leq \sqn{\hat{x}_t - x_\ast} &+ 2 \gamma L V_t + \frac{2 \gamma^2 \sigmaopt^2}{M} \\
            &+ \frac{2 \gamma}{M} \sum_{m=1}^{M} \br{ -\frac{1}{4} \br{f(x_t^m) - f(x_\ast)} - \frac{\mu}{2} \sqn{x_t^m - x_\ast} }.
        \end{split}
    \end{align}
    Applying Jensen's inequality from Proposition~\ref{pr:jensen} to $\frac{1}{4} \br{f(x_t^m)- f(x_\ast)} + \frac{\mu}{2} \sqn{x_t^m - x_\ast}$, we obtain
    \begin{align}
        \label{eq:lma-ogc-proof-3}
        - \avemm \br{\frac{1}{4} \br{f(x_t^m)- f(x_\ast)} + \frac{\mu}{2} \sqn{x_t^m - x_\ast} } \overset{\eqref{eq:jensen}}{\leq} - \br{ \frac{1}{4} \br{f(\hat{x}_t) - f(x_\ast)} + \frac{\mu}{2} \sqn{\hat{x}_t - x_\ast}}.
    \end{align}
    Plugging~\eqref{eq:lma-ogc-proof-3} in \eqref{eq:lma-ogc-proof-2}, we get
    \begin{align*}
        \ecn{\hat{x}_{t+1} - x_\ast} \overset{\eqref{eq:lma-ogc-proof-2}, \eqref{eq:lma-ogc-proof-3}}{\leq} \br{1 - \gamma \mu} \sqn{\hat{x}_t - x_\ast} &+ 2 \gamma L V_t + \frac{2 \gamma^2 \sigmaopt^2}{M} - \frac{\gamma}{2} \br{f(\hat{x}_t) - f(x_\ast)},
    \end{align*}
    which is the claim of this lemma.
\end{proof}

\begin{lemma}
    \label{lemma:average-gradient-deviation}
    Bounding the deviation of the gradients from their average: under Assumptions~\ref{asm:convexity-and-smoothness} and \ref{asm:finite-sum-stochastic-gradients} for identical data we have for all $t \geq 0$,
    \begin{equation}
        \label{eq:lma-average-gradient-deviation}
        \avemm \ecn{g_t^m - \avemm g_t^m} \leq 2 \sigmaopt^2 + \frac{4L}{M} \sum_{m=1}^{M} D_{f} (x_t^m, x_\ast).
    \end{equation}
\end{lemma}
\begin{proof}
    We start by the variance bound,
    \begin{align*}
        \avemm \sqn{g_t^m - \avemm g_t^m} &\overset{\eqref{eq:variance_m}}{=} \avemm \sqn{g_t^m} - \sqn{\avemm g_t^m} \\
        &\leq \avemm \sqn{g_t^m}.
    \end{align*}
    We now take expectations and use Lemma \ref{lemma:gradient-smoothness-bound}:
    \begin{align*}
        \avemm \ecn{g_t^m - \avemm g_t^m} &\leq \avemm \ecn{g_t^m} \\
        &\leq \frac{1}{M} \sum_{m=1}^{M} \br{2 \sigma_m^2 + 4L D_{f} (x_t^m, x_\ast)} \\
        &= 2 \sigmaopt^2 + \frac{4 L}{M} \sum_{m=1}^{M} D_{f} (x_t^m, x_\ast).
    \end{align*}
\end{proof}

\begin{lemma}
    \label{lemma:average-iterate-deviation-recursion}
    Suppose that Assumptions~\ref{asm:convexity-and-smoothness}, and \ref{asm:finite-sum-stochastic-gradients} hold for identical data. Choose $\gamma \leq \frac{1}{2 L}$, then for all $t \geq 0$:
    \begin{equation}
        \label{eq:lma-average-iterate-deviation-recursion}
        \ec{V_{t+1}} \leq \br{1 - \gamma \mu} \ec{V_t} + 2 \gamma \ec{D_{f} (\hat{x}_t, x_\ast)} + 2 \gamma^2 \sigmaopt^2.
    \end{equation}
\end{lemma}
\begin{proof}
    If $t + 1 = t_p$ for some $p \in \N$ then the left hand side is zero and the above inequality is trivially satisfied. If not, then recall that $x_{t+1}^m = x_t^m - \gamma g_t^m$ and $\hat{x}_{t+1} = \hat{x}_t - \gamma g_t$ where $\ec{g_t^m} = \nabla f(x_t^m)$ and $g_t = \avemm g_t^m$. Hence, for the expectation conditional on $(x_t^m)_{m=1}^{M}$ we have
    \begin{align*}
        \ecn{x_{t+1}^m - \hat{x}_{t+1}} &= \ecn{x_t^m - \hat{x}_t - \gamma \br{g_t^m - g_t}} \\
        &= \ecn{x_t^m - \hat{x}_t^m} + \gamma^2 \ecn{g_t^m - g_t} - 2 \gamma \ec{\ev{x_t^m - \hat{x}_t, g_t^m - g_t}} \\
        &= \ecn{x_t^m - \hat{x}_t^m} + \gamma^2 \ecn{g_t^m - g_t} - 2 \gamma \ev{x_t^m - \hat{x}_t, \nabla f(x_t^m) - \bar{g}_t},
    \end{align*}
    where $\bar{g}_t = \ec{\avemm g_t^m} = \avemm \nabla f(x_t^m)$. Averaging over $m$ in the last equality,
    \begin{align}
        \ec{V_{t+1}} &= V_t + \frac{\gamma^2}{M} \sum_{m=1}^{M}\ecn{g_t^m - g_t} - \frac{2\gamma}{M} \sum_{m=1}^{M} \ev{x_t^m - \hat{x}_t, \nabla f(x_t^m)} + \frac{2 \gamma}{M} \sum_{m=1}^{M} \ev{x_t^m - \hat{x}_t, \bar{g}_t} \nonumber \\
        &= V_t + \frac{\gamma^2}{M} \sum_{m=1}^{M} \ecn{g_t^m - g_t} - \frac{2\gamma}{M} \sum_{m=1}^{M} \ev{x_t^m - \hat{x}_t, \nabla f(x_t^m)} + 2 \gamma \underbrace{\ev{\hat{x}_t - \hat{x}_t, \bar{g}_t}}_{=0} \nonumber \\
        &= V_t + \frac{\gamma^2}{M} \sum_{m=1}^{M} \ecn{g_t^m - g_t} - \frac{2 \gamma}{M} \sum_{m=1}^{M} \ev{x_t^m - \hat{x}_t, \nabla f(x_t^m)}. \label{eq:lma-aidr-proof-1}
    \end{align}
    We now use Lemma~\ref{lemma:average-gradient-deviation} to bound the second term in \eqref{eq:lma-aidr-proof-1},
    \begin{align}
        \label{eq:lma-aidr-proof-2}
        \ec{V_{t+1}} &\overset{\eqref{eq:lma-average-gradient-deviation}}{\leq} V_t + \frac{4 L \gamma^2}{M} \sum_{m=1}^{M} D_{f} (x_t^m, x_\ast) + 2 \gamma^2 \sigmaopt^2 - \frac{2 \gamma}{M} \sum_{m=1}^{M} \ev{x_t^m - \hat{x}_t, \nabla f(x_t^m)}.
    \end{align}
    We now use Assumption~\ref{asm:convexity-and-smoothness} to bound the last term in \eqref{eq:lma-aidr-proof-2}:
    \begin{align}
        \label{eq:lma-aidr-proof-3}
        \ev{\hat{x}_t - x_t^m, \nabla f(x_t^m)} &\overset{\eqref{eq:asm-strong-convexity}}{\leq} f(\hat{x}_t) - f(x_t^m) - \frac{\mu}{2} \sqn{x_t^m - \hat{x}_t}.
    \end{align}
    Plugging \eqref{eq:lma-aidr-proof-3} into \eqref{eq:lma-aidr-proof-2},
    \begin{align}
        \label{eq:lma-aidr-proof-4}
        \ec{V_{t+1}} &\leq \br{1 - \gamma \mu} V_t + 2 \gamma^2 \sigmaopt^2 + \frac{4 L \gamma^2}{M} \sum_{m=1}^{M} \br{f(x_t^m) - f(x_\ast)} + \frac{2 \gamma}{M} \sum_{m=1}^{M} \br{f(\hat{x}_t) - f(x_t^m)}.
    \end{align}
    Using that $\gamma \leq \frac{1}{2L}$ in \eqref{eq:lma-aidr-proof-4},
    \begin{align*}
        \ec{V_{t+1}} &\leq \br{1 - \gamma \mu} V_t + 2 \gamma^2 \sigmaopt^2 + \frac{2 \gamma}{M} \sum_{m=1}^{M} \br{f(x_t^m) - f(x_\ast) + f(\hat{x}_t) - f(x_t^m)} \\
        &= \br{1 - \gamma \mu} V_t + 2 \gamma^2 \sigmaopt^2 + 2\gamma \br{f(\hat{x}_t) - f(x_\ast)}.
    \end{align*}
    Taking unconditional expectations and using the tower property yields the lemma's statement.
\end{proof}

\begin{lemma}{Epoch Iterate Deviation Bound}
    \label{lemma:epoch-iterate-deviation-bound}
    Suppose that Assumptions~\ref{asm:convexity-and-smoothness}, and \ref{asm:finite-sum-stochastic-gradients} hold with identical data. Assume that Algorithm~\ref{alg:local_sgd} is run with stepsize $\gamma > 0$, let $p \in \N$ be such that $t_{p}$ is a synchronization point then for $v=t_{p+1} - 1$ we have for $\alpha \eqdef 1 - \gamma \mu$,
    \begin{align*}
        \sum_{t=t_p}^{v} \alpha^{v-t} \cdot \ec{V_t} &\leq \frac{2 \gamma \br{H-1}}{\alpha} \sum_{t=t_p}^{v} \alpha^{v-t} \cdot \ec{D_f (\hat{x}_t, x_\ast)} + 2 \gamma^2 \sigmaopt^2 \br{H - 1} \sum_{t=t_p}^{v} \alpha^{v-t}. 
    \end{align*}
\end{lemma}
\begin{proof}
    We start with Lemma~\ref{lemma:average-iterate-deviation-recursion},
    \begin{align*}
        \ec{V_{t}} &\leq \br{1 - \gamma \mu} \ec{V_{t-1}} + 2 \gamma \ec{D_{f} (\hat{x}_{t-1}, x_\ast)} + 2 \gamma^2 \sigmaopt^2 \\
        &= \alpha \cdot \ec{V_{t-1}} + 2 \gamma \ec{D_{f} (\hat{x}_{t-1}, x_\ast)} + 2 \gamma^2 \sigmaopt^2.
    \end{align*}
    By assumption there is some synchronization point $p \in \N$ such that $t_p \leq t \leq t_{p+1} - 1$ and $t_{p+1} - t_p \leq H$, recursing the above inequality until $t_{p}$ and using that $V_{t_p} = 0$,
    \begin{align}
        \ec{V_{t}} \leq \alpha^{t - t_p} \ec{V_{t_p}} &+ 2 \gamma \sum_{j=t_p}^{t-1} \alpha^{t - j - 1} \ec{D_{f} (\hat{x}_{j}, x_\ast)} \nonumber \\
        &+ 2 \sigmaopt^2 \sum_{j=t_p}^{t-1} \gamma^2 \alpha^{t - 1 - j} \nonumber \\
        \label{eq:lma-eidb-1}
        = \frac{2 \gamma}{\alpha} \sum_{j=t_p}^{t-1} \alpha^{t - j} \ec{D_{f} (\hat{x}_{j}, x_\ast)} &+ 2 \gamma^2 \sigmaopt^2 \sum_{j=t_p}^{t-1} \alpha^{t- 1 - j}.
    \end{align}
    The second term in \eqref{eq:lma-eidb-1} can be bounded as follows: because $\alpha \leq 1$ then $\alpha^{t - 1 - j} \leq 1$ for $j \leq t-1$, hence for $t \leq t_{p+1} - 1$
    \begin{align}
        2 \gamma^2 \sigmaopt^2 \sum_{j=t_p}^{t-1} \alpha^{t - 1 - j} &\leq 2 \gamma^2 \sigmaopt^2 \sum_{j=t_p}^{t-1} 1 \nonumber \\
        &= 2 \gamma^2 \sigmaopt^2 \br{t - t_p} \nonumber \\
        &\leq 2 \gamma^2 \sigmaopt^2 \br{t_{p+1} - t_p - 1} \nonumber \\
        \label{eq:lma-eidb-2}
        &\leq 2 \gamma^2 \sigmaopt^2 \br{H - 1}.
    \end{align}
    Using \eqref{eq:lma-eidb-2} in \eqref{eq:lma-eidb-1},
    \begin{align}
        \label{eq:lma-eidb-3}
        \ec{V_t} &\leq \frac{2 \gamma}{\alpha} \sum_{j=t_p}^{t-1} \alpha^{t - j} \ec{D_{f} (\hat{x}_{j}, x_\ast)} + 2 \gamma^2 \sigmaopt^2 \br{H - 1}.
    \end{align}
    Then summing up \eqref{eq:lma-eidb-3} weighted by $\alpha^{v-t}$ for $v = t_{p+1} - 1$,
    \begin{align}
        \begin{split}
            \sum_{t=t_p}^{v} \alpha^{v-t} \ec{V_t} &\leq \frac{2 \gamma}{\alpha} \sum_{t=t_p}^{v} \alpha^{v-t} \sum_{j=t_p}^{t-1} \alpha^{t-j} \ec{D_f (\hat{x}_j, x_\ast)} + 2 \gamma^2 \sigmaopt^2 \br{H - 1} \sum_{t=t_p}^{v} \alpha^{v-t}
        \end{split}\nonumber \\
        \label{eq:lma-eidb-4}
&=\frac{2 \gamma}{\alpha} \sum_{t=t_p+1}^{v} \alpha^{v-t} \sum_{j=t_p}^{t-1} \alpha^{t-j} \ec{D_f (\hat{x}_j, x_\ast)} + 2 \gamma^2 \sigmaopt^2 \br{H - 1} \sum_{t=t_p}^{v} \alpha^{v-t},
    \end{align}
     where we used that at $t=t_p$ the sum $\sum_{j=t_p}^{t-1} \alpha^{t-j} \ec{D_f (\hat{x}_j, x_\ast)}$ is zero. Then by adding more Bregman divergence terms (which are positive) to the inner sum we obtain
    \begin{align}
        \sum_{t=t_p+1}^{v} \alpha^{v-t} \sum_{j=t_p}^{t-1} \alpha^{t-j} \ec{D_f (\hat{x}_j, x_\ast)} &\leq \sum_{t=t_p+1}^{v} \alpha^{v-t} \sum_{j=t_p}^{v} \alpha^{t-j} \ec{D_f (\hat{x}_j, x_\ast)} \nonumber \\
        &= \sum_{t=t_p+1}^{v} \sum_{j=t_p}^{v} \alpha^{v - j} \ec{D_f (\hat{x}_j, x_\ast)} \nonumber \\
        &= \br{v - t_p} \sum_{j=t_p}^{v} \alpha^{v - j} \ec{D_f (\hat{x}_j, x_\ast)} \nonumber \\
        &= \br{t_{p+1} - t_p - 1} \sum_{j=t_p}^{v} \alpha^{v - j} \ec{D_f (\hat{x}_j, x_\ast)} \nonumber \\
        \label{eq:lma-eidb-5}
        &\leq \br{H-1} \sum_{j=t_p}^{v} \alpha^{v - j} \ec{D_f (\hat{x}_j, x_\ast)}.
    \end{align}
    Combining \eqref{eq:lma-eidb-5} and \eqref{eq:lma-eidb-4} we have,
    \begin{align*}
        \sum_{t=t_p}^{v} \alpha^{v-t} \cdot \ec{V_t} &\leq \frac{2 \gamma \br{H-1}}{\alpha} \sum_{j=t_p}^{v} \alpha^{v-j} \cdot \ec{D_f (\hat{x}_j, x_\ast)} + 2 \gamma^2 \sigmaopt^2 \br{H - 1} \sum_{t=t_p}^{v} \alpha^{v-t}. 
    \end{align*}
    Finally, renaming the variable $j$ gives us the claim of this lemma.
\end{proof}

\subsection{Proof of Theorem~\ref{theorem:sc-unbounded-variance-iid}}
\begin{proof}
    Let $(t_{p})_{p=1}^{\infty}$ index all the times $t$ at which communication and averaging happen. Taking expectations in Lemma~\ref{lemma:optimality-gap-contraction} and letting $r_{t} = \hat{x}_t - x_\ast$,
    \begin{align}
        \label{smooth-case-thm-proof-3}
        \ecn{r_{t+1}} &\leq \br{1 - \gamma \mu} \ecn{r_t} + 2 \gamma L \ec{V_{t}} + \frac{2 \gamma^2 \sigmaopt^2}{M} - \frac{\gamma}{2} \ec{D_{f} (\hat{x}_t, x_\ast)} \\
        &= \br{1 - \gamma \mu} \ecn{r_t} + \br{2 \gamma L \ec{V_t} - \frac{\gamma}{2} D_{f} (\hat{x}_t, x_\ast)} + \frac{2 \gamma^2 \sigmaopt^2}{M}.
    \end{align}
    Let $T = t_p - 1$ for some $p \in \N$, then expanding out $\ecn{r_t}$ in \eqref{smooth-case-thm-proof-3},
    \begin{align}
        \ecn{r_{T+1}} &\leq \br{1 - \gamma \mu}^{T+1} \ecn{\hat{x}_{0} - x_\ast} + \sum_{t=0}^{T} \br{1 - \gamma \mu}^{T-i} \frac{2 \gamma^2 \sigmaopt^2}{M} \nonumber \\
        &+ \sum_{i=0}^{T} \br{1 - \gamma \mu}^{T-i} \br{2 \gamma L \ec{V_i} - \frac{\gamma}{2} D_{f} (\hat{x}_i, x_\ast)} \nonumber \\
        \label{smooth-case-thm-proof-4}
        &\leq \br{1 - \gamma \mu}^{T+1} \ecn{x_0 - x_\ast} + \frac{2 \gamma \sigmaopt^2}{\mu M} + \frac{\gamma}{2} \sum_{i=0}^{T} \br{1 - \gamma \mu}^{T-i} \ec{4 L V_i - D_{f} (\hat{x}_i, x_\ast)}.
    \end{align}
    It remains to bound the last term in \eqref{smooth-case-thm-proof-4}. We have
    \begin{align}
        \sum_{i=0}^{T} \br{1 - \gamma \mu}^{T-i} (4 L \ec{V_i} &- D_{f} (\hat{x}_i, x_\ast)) \notag \\
        &= \sum_{k=1}^{p} \sum_{i=t_{k-1}}^{t_k - 1} \br{1 - \gamma \mu}^{T-i} \br{4 L \ec{V_i} - D_{f} \br{\hat{x}_i, x_\ast}} \nonumber \\
        \label{smooth-case-thm-proof-5}
        &= \sum_{k=1}^{p} \br{1 - \gamma \mu}^{T - \br{t_k - 1}} \sum_{t_{k-1}}^{t_k - 1} \br{1 - \gamma \mu}^{t_k - 1 - i} \ec{4 L V_i - D_{f} \br{\hat{x}_i, x_\ast}},
    \end{align}
    where in the first line we just count $i$ by decomposing it over all the communication intervals. Fix $k \in \N$ and let $v_k = t_{k} - 1$. Then by Lemma~\ref{lemma:epoch-iterate-deviation-bound} we have,
    \begin{align}
        \sum_{i=t_k}^{v_k} \br{1 - \gamma \mu}^{v_k - i} \ec{V_i} &\leq \frac{2 \gamma (H - 1)}{\alpha} \sum_{i=t_k}^{v_k} \alpha^{v_k - i} \ec{D_{f} (\hat{x}_i, x_\ast)} + \sum_{i=t_k}^{v_k} \alpha^{v_k - i} 2 \gamma^2 \sigma^2 (H - 1),
        \label{smooth-case-thm-proof-2}
    \end{align}
    where $\alpha = 1 - \gamma \mu$. Using \eqref{smooth-case-thm-proof-2} in \eqref{smooth-case-thm-proof-5},
    \begin{align}
        &4 L \sum_{i = t_{k-1}}^{v_k} \br{1 - \gamma \mu}^{v_k - i} \ec{V_s} - \sum_{i=t_{k-1}}^{v_k} \br{1 - \gamma \mu}^{v_k-i} \ec{D_{f} (\hat{x}_i, x_\ast)} \nonumber \\
        &\leq 4 L \br{  \frac{2 \gamma \br{H - 1}}{1 - \gamma \mu} \sum_{i=t_{k-1}}^{v_k} \br{1 - \gamma \mu}^{v_k-i} \ec{D_{f} (\hat{x}_i, x_\ast)} + \sum_{i=t_{k-1}}^{v_k} \br{1 - \gamma \mu}^{v_k - i} 2 \gamma^2 \sigmaopt^2 \br{H - 1} } \nonumber  \\
        &- \sum_{i=t_{k-1}}^{v_k} \br{1 - \gamma \mu}^{v_k-i} \ec{D_{f} (\hat{x}_i, x_\ast)} \nonumber \\
        &= \sum_{i=t_{k-1}}^{v_k} \br{1 - \gamma \mu}^{v_k - i} 8 \gamma^2 \sigmaopt^2 \br{H - 1} L - \sum_{i=t_{k-1}}^{v_k} \br{1 - \frac{8 \gamma L \br{H - 1}}{1 - \gamma \mu}} \br{1 - \gamma \mu}^{v_k-i} \ec{D_{f} (\hat{x}_i, x_\ast)} \nonumber \\
        \label{smooth-case-thm-proof-6}
        &\leq \sum_{i=t_{k-1}}^{v_k} \br{1 - \gamma \mu}^{v_k - i} 8 \gamma^2 \sigmaopt^2 \br{H - 1} L,
    \end{align}
    where in in the third line we used that our choice of $\gamma$ guarantees that $1 - \frac{8 \gamma L H}{1 - \gamma \mu} \geq 0$. Using \eqref{smooth-case-thm-proof-6} in \eqref{smooth-case-thm-proof-5},
    \begin{align}
        \sum_{i=0}^{T} \br{1 - \gamma \mu}^{T-i} &\ec{4 L V_i - D_{f} (\hat{x}_i, x_\ast)} \nonumber \\
        &\leq \sum_{k=1}^{p} \br{1 - \gamma \mu}^{T - \br{t_k - 1}} \sum_{i= t_{k-1}}^{t_k - 1} \br{1 - \gamma \mu}^{t_k - 1 - i} \ec{4 L V_i - D_{f} \br{\hat{x}_i, x_\ast}} \nonumber \\
        &\leq \sum_{k=1}^{p} \br{1 - \gamma \mu}^{T - \br{t_k - 1}} \sum_{i=t_{k-1}}^{t_k - 1} \br{1 - \gamma \mu}^{t_k - 1 - i} 8 \gamma^2 \sigmaopt^2 \br{H - 1} L \nonumber \\
        &= \sum_{k=1}^{p} \sum_{i=t_{k-1}}^{t_k - 1} \br{1 - \gamma \mu}^{T - i} 8 \gamma^2 \sigma^2 \br{H - 1} L \nonumber \\
        &= \sum_{i=0}^{T} \br{1 - \gamma \mu}^{T - i} 8 \gamma^2 \sigmaopt^2 \br{H - 1} L \nonumber \\
        \label{smooth-case-thm-proof-7}
        &\leq \frac{8 \sigmaopt^2 \gamma \br{H - 1} L}{\mu}.
    \end{align}
    Using \eqref{smooth-case-thm-proof-7} in \eqref{smooth-case-thm-proof-4},
    \begin{align*}
        \ecn{r_{T+1}} &\leq \br{1 - \gamma \mu}^{T+1} \ecn{x_0 - x_\ast} + \frac{2 \gamma \sigmaopt^2}{\mu M} + \frac{\gamma}{2} \sum_{i=0}^{T} \br{1 - \gamma \mu}^{T-i} \ec{4 L V_i - D_{f} (\hat{x}_i, x_\ast)} \\
        &\leq \br{1 - \gamma \mu}^{T + 1} \ecn{x_0 - x_\ast} + \frac{2 \gamma \sigmaopt^2}{\mu M} + \frac{4 \sigmaopt^2 \gamma^2 \br{ H - 1} L}{\mu},
    \end{align*}
    which is the claim of the theorem.
\end{proof}

\subsection{Proof of Theorem~\ref{thm:wc-iid-unbounded-var}}
\begin{proof}
    Start with Lemma~\ref{lemma:optimality-gap-contraction} with $\mu = 0$, then the conditional expectations satisfies
    \begin{align*}
        \ecn{\hat{x}_{t+1} - x_\ast} &\overset{\eqref{eq:lma-optimality-gap-contraction}}{\leq} \sqn{\hat{x}_t - x_\ast} + 2 \gamma L V_t + \frac{2 \gamma^2 \sigmaopt^2}{M} - \frac{\gamma}{2} D_{f} (\hat{x}_t, x_\ast).
    \end{align*}
    Taking full expectations and rearranging,
    \begin{align*}
        \frac{\gamma}{2} \ec{D_{f} (\hat{x}_t, x_\ast)} &\leq \ecn{\hat{x}_t - x_\ast} - \ecn{\hat{x}_{t+1} - x_\ast} + 2 \gamma L \ec{V_t} + \frac{2 \gamma^2 \sigmaopt^2}{M}.
    \end{align*}
    Averaging as $t$ varies from $0$ to $T-1$,
    \begin{align}
        \frac{\gamma}{2 T} \sum_{t=0}^{T-1} \ec{D_{f} (\hat{x}_t, x_\ast)} &\leq \frac{1}{T} \sum_{t=0}^{T-1} \br{\ecn{\hat{x}_t - x_\ast} - \ecn{\hat{x}_{t+1} - x_\ast}} + \frac{2 \gamma L}{T} \sum_{t=0}^{T-1} \ec{V_t} + \frac{2 \gamma^2 \sigmaopt^2}{M} \nonumber \\
        &= \frac{1}{T} \br{\sqn{x_0 - x_\ast} - \ecn{\hat{x}_{T} - x_\ast}} + \frac{2 \gamma L}{T} \sum_{t=0}^{T-1} \ec{V_t} + \frac{2 \gamma^2 \sigmaopt^2}{M} \nonumber \\
        \label{eq:thm-wcc-proof-1}
        &\leq \frac{\sqn{x_0 - x_\ast}}{T} + \frac{2 \gamma L}{T} \sum_{t=0}^{T-1} \ec{V_t} + \frac{2 \gamma^2 \sigmaopt^2}{M}.
    \end{align}
    To bound the sum of deviations in \eqref{eq:thm-wcc-proof-1}, we use Lemma~\ref{lemma:epoch-iterate-deviation-bound} with $\mu = 0$ (and noticing that because $\mu = 0$ we have $\alpha = 1$),
    \begin{align}
        \label{eq:thm-wcc-proof-2}
        \sum_{t=t_p}^{t_{p+1} - 1} \ec{V_t} &\leq \sum_{t=t_p}^{t_{p+1} - 1} \br{2 \gamma (H - 1) \ec{D_f (\hat{x}_t, x_\ast)} + 2 \gamma^2 \sigmaopt^2 (H - 1) }.
    \end{align}
    Since by assumption $T$ is a synchronization point, then there is some $k \in \N$ such that $T = t_{k}$. To estimate the sum of deviations in \eqref{eq:thm-wcc-proof-1} we use double counting to decompose it over each epoch, use \eqref{eq:thm-wcc-proof-2}, and then use double counting again:
    \begin{align}
        \sum_{t=0}^{T - 1} \ec{V_t} &= \sum_{p=0}^{k-1} \sum_{t=t_{p}}^{t_{p+1} - 1} \ec{V_t} \nonumber \\
        &\overset{\eqref{eq:thm-wcc-proof-2}}{\leq} \sum_{p=0}^{k-1} \sum_{t=t_{p}}^{t_{p+1} - 1} \br{2 \gamma H \ec{D_{f} (\hat{x}_t, x_\ast)} + 2 \gamma^2 \sigmaopt^2 \br{H - 1}} \nonumber \\
        \label{eq:thm-wcc-proof-3}
        &= \sum_{t=0}^{T-1} \br{2 \gamma \br{H - 1} \ec{D_{f} (\hat{x}_t, x_\ast)} + 2 \gamma^2 \sigmaopt^2 (H - 1)}.
    \end{align}
    Using \eqref{eq:thm-wcc-proof-3} in \eqref{eq:thm-wcc-proof-1} and rearranging we get,
    \begin{align*}
        \frac{\gamma}{2T} \sum_{t=0}^{T-1} \ec{D_{f} (\hat{x}_t, x_\ast)} \leq \frac{\sqn{x_0 - x_\ast}}{T} &+ \frac{2 \gamma L}{T} \sum_{t=0}^{T-1} \br{2 \gamma (H-1) \ec{D_{f} (\hat{x}_t, x_\ast)} + 2 \gamma^2 \sigmaopt^2 (H-1)} + \frac{2 \gamma^2 \sigmaopt^2}{M} \\
        \frac{\gamma}{2T} \sum_{t=0}^{T-1} \br{1 - 8 \gamma (H-1) L} \ec{D_{f} (\hat{x}_t, x_\ast)} &\leq \frac{\sqn{x_0 - x_\ast}}{T} + \frac{2 \gamma^2 \sigmaopt^2}{M} + 4 \gamma^3 L \sigmaopt^2 (H-1).	
    \end{align*}
    By our choice of $\gamma$ we have that $1 - 8 \gamma L (H-1) \geq \frac{2}{10}$, using this with some algebra we get
    \begin{align*}
        \frac{\gamma}{10 T} \sum_{t=0}^{T-1} \ec{D_{f} (\hat{x}_t, x_\ast)} &\leq \frac{\sqn{x_0 - x_\ast}}{T} + \frac{2 \gamma^2 \sigmaopt^2}{M} + 4 \gamma^3 L \sigmaopt^2 (H-1).
    \end{align*}
    Dividing both sides by $\gamma/10$ and using Jensen's inequality yields the theorem's claim.
\end{proof}

\section{Proofs for Heterogeneous data}
\subsection{Preliminary Lemmas}
\begin{lemma}
    \label{lemma:average-gradient-bound}
    Suppose that Assumptions~\ref{asm:convexity-and-smoothness} and \ref{asm:finite-sum-stochastic-gradients} hold with $\mu \geq 0$ for heterogeneous data. Then for expectation conditional on $x_t^1, x_t^2, \ldots, x_t^m$ and for $M \geq 2$, we have
    \begin{align}
        \label{eq:lma-average-gradient-bound}
        \ecn{g_t} \leq 2 L^2 V_t + 8 L D_{f} (\hat{x}_t, x_\ast) + \frac{4 \sigmaf^2}{M}.
    \end{align}
\end{lemma}
\begin{proof}
    Starting with the left-hand side,
    \begin{align}
        \label{eq:lma-agb-proof-1}
        \ecn{g_t} &\overset{\eqref{eq:sum_sqnorm}}{\leq} 2 \ecn{g_t - \frac{1}{M} \sum_{m=1}^{M} \nabla f_m (\hat{x}_t, z_m) } + 2 \ecn{ \frac{1}{M} \sum_{m=1}^{n} \nabla f_m (\hat{x}_t, z_m) }.
    \end{align}
    To bound the first term in \eqref{eq:lma-agb-proof-1} we have that using the $L$-smoothness of $f_m (\cdot, z_m)$,
    \begin{align}
        2 \ecn{g_t - \frac{1}{M} \sum_{m=1}^{M} \nabla f_m (\hat{x}_t, z_m) } &= 2 \ecn{ \frac{1}{M} \sum_{m=1}^{M}  \nabla f_m (x_t^m, z_m) - \nabla f_m (\hat{x}_t, z_m)} \nonumber
        \\
        &\leq \frac{2}{M} \sum_{m=1}^{M} \ecn{\nabla f_m (x_t^m, z_m) - \nabla f_m (\hat{x}_t, z_m)} \nonumber \\
        \label{eq:lma-agb-proof-2}
        &\leq  \frac{2 L^2}{M} \sum_{m=1}^{M} \sqn{x_t^m - \hat{x}_t}.
    \end{align}
    and where in the second inequality we have used Jensen's inequality and the convexity of the map $x\mapsto \|x\|^2$. For the second term in \eqref{eq:lma-agb-proof-1}, we have
    \begin{align}
        \label{eq:lma-agb-proof-2-2}
        \begin{split}
            \ecn{ \frac{1}{M} \sum_{m=1}^{M} \nabla f_m (\hat{x}_t, z_m) } &\overset{\eqref{eq:variance_def}}{=} \ecn{ \frac{1}{M} \sum_{m=1}^{M} \nabla f_m (\hat{x}_t, z_m) - \frac{1}{M} \sum_{m=1}^{M} \nabla f_m (\hat{x}_t)} \\
            &+ \sqn{ \frac{1}{M} \sum_{m=1}^{M} \nabla f_m (\hat{x}_t)}.
        \end{split}
    \end{align}
    For the first term in \eqref{eq:lma-agb-proof-2-2} we have by the independence of $z_1, z_2, \ldots, z_m$,
    \begin{align*}
        \ecn{ \frac{1}{M} \sum_{m=1}^{M} \nabla f_m (\hat{x}_t, z_m) - \frac{1}{M} \sum_{m=1}^{M} \nabla f_m (\hat{x}_t)} &= \frac{1}{M^2} \sum_{m=1}^{M} \ecn{\nabla f_m (\hat{x}_t, z_m) - \nabla f_m (\hat{x}_t)} \\
        &\overset{\eqref{eq:variance_sqnorm_upperbound}}{\leq} \frac{1}{M^2} \sum_{m=1}^{M} \ecn{\nabla f_m (\hat{x}_t, z_m)} \\
        \overset{\eqref{eq:sum_sqnorm}}{\leq} \frac{2}{M^2} \sum_{m=1}^{M} \ecn{\nabla f_m (\hat{x}_t, z_m) - \nabla f_m (x_\ast, z_m)} &+ \frac{2}{M^2} \sum_{m=1}^{M} \ecn{\nabla f_m (x_\ast, z_m)} \\
        &\overset{\eqref{eq:bregman-dif}}{\leq} \frac{4 L}{M^2} \sum_{m=1}^{M} D_{f_m} (\hat{x}_t, x_\ast) + \frac{2 \sigmaf^2}{M}\\
        &= \frac{4L}{M} D_{f} (\hat{x}_t, x_\ast) + \frac{2 \sigmaf^2}{M}.
    \end{align*}
    Using this in \eqref{eq:lma-agb-proof-2-2} we have,
    \begin{align*}
         \ecn{ \frac{1}{M} \sum_{m=1}^{M} \nabla f_m (\hat{x}_t, z_m) } &\leq \frac{4L}{M} D_{f} (\hat{x}_t, x_\ast) + \frac{2 \sigmaf^2}{M} + \ecn{ \frac{1}{M} \sum_{m=1}^{M} \nabla f_m (\hat{x}_t)} \\
         &= \frac{4L}{M} D_{f} (\hat{x}_t, x_\ast) + \frac{2 \sigmaf^2}{M} + \sqn{ \nabla f(\hat{x}_t)}.
    \end{align*}
    Now notice that
    \[ \sqn{\nabla f(\hat{x}_t)} = \sqn{\nabla f(\hat{x}_t) - \nabla f(x_\ast)} \leq 2 L D_{f} (\hat{x}_t, x_\ast).  \]
    Using this in the previous inequality we have,
    \begin{align*}
        \ecn{\frac{1}{M} \sum_{m=1}^{M} \nabla f_m (\hat{x}_t, z_m)} \leq 2L \br{1 + \frac{2}{M}} D_{f} (\hat{x}_t, x_\ast) + \frac{2 \sigmaf^2}{M}.
    \end{align*}
    Because $M \geq 2$ we have $1 + \frac{2}{M} \leq 2$, hence
    \begin{align}
        \label{eq:lma-agb-proof-3}
        \ecn{\frac{1}{M} \sum_{m=1}^{M} \nabla f_m (\hat{x}_t, z_m)} \leq 4 L D_{f} (\hat{x}_t, x_\ast) + \frac{2 \sigmaf^2}{M}.
    \end{align}
    Combining \eqref{eq:lma-agb-proof-2} and \eqref{eq:lma-agb-proof-3} in \eqref{eq:lma-agb-proof-1} we have,
    \begin{align*}
        \ecn{g_t} \leq 2 L^2 V_t + 8 L D_{f} (\hat{x}_t, x_\ast) + \frac{4 \sigmaf^2}{M}.
    \end{align*}
\end{proof}

\begin{lemma}
    \label{lemma:inner-product-bound}
    Suppose that Assumption~\ref{asm:convexity-and-smoothness} holds with $\mu \geq 0$ for heterogeneous data (holds for each $f_m$ for $m = 1, 2, \ldots, M$). Then we have,
    \begin{equation}
        \label{eq:lma-inner-product-bound}
        -\frac{2}{M} \sum_{m=1}^{M} \ev{\hat{x}_t - x_\ast, \nabla f_m (x_t^m)} \leq - 2  D_{f} (\hat{x}_t, x_\ast) -  \mu \sqn{ \hat{x}_t - x_\ast} +  L V_t.
    \end{equation}
\end{lemma}
\begin{proof}
    Starting with the left-hand side,
    \begin{align}
        \label{eq:lma-inner-prod-proof-1}
        - 2  \ev{\hat{x}_t - x_\ast, \nabla f_m (x_t^m)} &= - 2  \ev{x_t^m - x_\ast, \nabla f_m (x_t^m)} - 2  \ev{\hat{x}_t - x_t^m, \nabla f_m (x_t^m)}.
    \end{align}
    The first term in \eqref{eq:lma-inner-prod-proof-1} is bounded by strong convexity:
    \begin{align}
        \label{eq:lma-inner-prod-proof-2}
        - \ev{x_t^m - x_\ast, \nabla f_m (x_t^m)} &\leq f_m (x_\ast) - f_m (x_t^m) - \frac{\mu}{2} \sqn{x_t^m - x_\ast}.
    \end{align}
    For the second term, we use $L$-smoothness,
    \begin{align}
        \label{eq:lma-inner-prod-proof-3}
        - \ev{\hat{x}_t  - x_t^m, \nabla f_m (x_t^m)} \leq f_m (x_t^m) - f_m (\hat{x}_t) + \frac{L}{2} \sqn{x_t^m - \hat{x}_t}.
    \end{align}
    Combining \eqref{eq:lma-inner-prod-proof-3} and \eqref{eq:lma-inner-prod-proof-2} in \eqref{eq:lma-inner-prod-proof-1},
    \begin{align*}
        - 2  \ev{\hat{x}_t - x_\ast, \nabla f_m (x_t^m)} &\leq 2  \br{f_m (x_\ast) - f_m (x_t^m) - \frac{\mu}{2} \sqn{x_t^m - x_\ast}} \\
        &+ 2  \br{f_m (x_t^m) - f_m (\hat{x}_t) + \frac{L}{2} \sqn{x_t^m - \hat{x}_t}} \\
        &= 2  \br{f_m (x_\ast) - f_m (\hat{x}_t) - \frac{\mu}{2} \sqn{x_t^m - x_\ast} + \frac{L}{2} \sqn{x_t^m - \hat{x}_t}}.
    \end{align*}
    Averaging over $m$,
    \begin{align*}
        -\frac{2 }{M} \sum_{m=1}^{M} \ev{\hat{x}_t - x_\ast, \nabla f_m (x_t^m)} &\leq - 2  \br{f(\hat{x}_t) - f(x_\ast)} - \frac{ \mu}{M} \sum_{m=1}^{M} \sqn{x_t^m - x_\ast} + \frac{ L}{M} \sum_{m=1}^{M} \sqn{x_t^m - \hat{x}_t}.
    \end{align*}
    Note that the first term is the Bregman divergence $D_{f} (\hat{x}_t, x_\ast)$, and using Jensen's inequality we have $- \frac{1}{M} \sum_{m=1}^{M} \sqn{x_t^m - x_\ast} \leq - \sqn{\hat{x}_t - x_\ast}$, hence
    \begin{align*}
        -\frac{2 }{M} \sum_{m=1}^{M} \ev{\hat{x}_t - x_\ast, \nabla f_m (x_t^m)} &\leq - 2  D_{f} (\hat{x}_t, x_\ast) -  \mu \sqn{\hat{x}_t - x_\ast} +  L V_t,
    \end{align*}
    which is the claim of this lemma.
\end{proof}

\begin{lemma}
    \label{lemma:iterate-deviation-epoch}
    Suppose that Assumptions~\ref{asm:convexity-and-smoothness} and \ref{asm:finite-sum-stochastic-gradients} hold for Algorithm~\ref{alg:local_sgd} with heterogeneous data and with $\sup_{p} \abs{t_p - t_{p+1}} \leq H$. Let $p \in \N$, then for $v = t_{p+1} - 1$ and $\gamma \leq \frac{1}{4 L \br{H - 1}}$ we have,
    \begin{equation}
        \label{eq:lma-iterate-deviation-epoch}
        \sum_{t=t_p}^{v} \ec{V_t} \leq 8 L \gamma^2 \br{H-1}^2 \sum_{k=t_p}^{v} \ec{D_{f} (\hat{x}_k, x_\ast)} + 4 \gamma^2 \br{H-1}^2 \sum_{k=t_p}^{v} \sigmaf^2.
    \end{equation}
\end{lemma}
\begin{proof}
    Let $t$ be such that $t_p \leq t \leq t_{p+1} - 1 = v$. From the definition of $V_t$,
    \begin{align*}
        \ec{V_t} &= \frac{1}{M} \sum_{m=1}^{M} \ecn{x_t^m - \hat{x}_t} \\
        &= \frac{1}{M} \sum_{m=1}^{M} \ecn{ \br{x_{t_p}^{m} - \gamma \sum_{k=t_p}^{t-1} g_k^m } - \br{x_{t_p} - \gamma \sum_{k=t_p}^{t-1} g_k } }.
    \end{align*}
    Using that $x_{t_p} = x_{t_p}^{m}$ for all $m$ we have,
    \begin{align*}
        \ec{V_t} &= \frac{\gamma^2}{M} \sum_{m=1}^{M} \ecn{ \sum_{k=t_p}^{t-1} g_k^m - g_k } \\
        &\overset{\eqref{eq:jensen_norm}}{\leq} \frac{\gamma^2 \br{t - t_p}}{M} \sum_{m=1}^{M} \sum_{k=t_p}^{t-1} \ecn{g_k^m - g_k} \\
        &\overset{\eqref{eq:variance_sqnorm_upperbound}}{\leq} \frac{\gamma^2 \br{t - t_p}}{M} \sum_{m=1}^{M} \sum_{k=t_p}^{t-1} \ecn{g_k^m} \\
        &\leq \frac{\gamma^2 \br{H - 1}}{M} \sum_{m=1}^{M}  \sum_{k=t_p}^{t-1} \ecn{g_k^m},
    \end{align*}
    where in the third line we used that because $g_k = \avemm g_k^m$ then $\avemm \ecn{g_k^m - g_k} \leq \avemm \ecn{g_k^m}$, and in the fourth line we used that $t - t_p \leq t_{p+1} - t_p - 1 \leq H - 1$.  Summing up as $t$ varies from $t_p$ to $v$,
    \begin{align*}
        \sum_{t=t_p}^{v} \ec{V_t} &\leq \sum_{t=t_p}^{v} \frac{\gamma^2 (H-1)}{M} \sum_{m=1}^{M}  \sum_{k=t_p}^{t-1} \ecn{g_k^m} \\
&= \sum_{t=t_p+1}^{v} \frac{\gamma^2 (H-1)}{M} \sum_{m=1}^{M}  \sum_{k=t_p}^{t-1} \ecn{g_k^m},
    \end{align*}
    where the second line is because at $t=t_p$ the inner sum is zero. Then by adding terms we have
    \begin{align}
        \sum_{t=t_p}^{v} \ec{V_t} &\leq \sum_{t=t_p+1}^{v} \frac{\gamma^2 (H-1)}{M} \sum_{m=1}^{M}  \sum_{k=t_p}^{t-1} \ecn{g_k^m} \nonumber \\
        &\leq \sum_{t=t_p+1}^{v} \frac{\gamma^2 (H-1)}{M} \sum_{m=1}^{M}  \sum_{k=t_p}^{v-1} \ecn{g_k^m} \nonumber \\
        &= \frac{\gamma^2 (H-1) \br{v - t_p}}{M} \sum_{m=1}^{M} \sum_{k=t_p}^{v - 1} \ecn{g_k^m} \nonumber \\
        &\leq \frac{\gamma^2 (H-1)^2}{M} \sum_{m=1}^{M} \sum_{k=t_p}^{v-1} \ecn{g_k^m} \nonumber \\
        \label{eq:lma-ide-proof-1}
        &\leq \frac{\gamma^2 (H-1)^2}{M} \sum_{m=1}^{M} \sum_{k=t_p}^{v} \ecn{g_k^m}.
    \end{align}
    To bound the gradient norm term in \eqref{eq:lma-ide-proof-1}, we have
    \begin{align}
        \label{eq:lma-ide-proof-2}
        \begin{split}
            \ecn{g_k^m} \leq 3 \ecn{g_k^m - \nabla f_m (\hat{x}_k, z_m)} &+ 3 \ecn{\nabla f_m (\hat{x}_k, z_m) - \nabla f_m (x_\ast, z_m)} \\
            &+ 3 \ecn{\nabla f_m (x_\ast, z_m)}.
        \end{split}
    \end{align}
    The first term in \eqref{eq:lma-ide-proof-2} can be bounded by smoothness:
    \begin{equation}
        \label{eq:lma-ide-proof-3}
        \ecn{g_k^m - \nabla f_m (\hat{x}_t, z_m)} = \ecn{\nabla f_m (x_k^m, z_m) - \nabla f_m (\hat{x}_t, z_m)} \leq L^2 \ecn{x_k^m - \hat{x}_k}.
    \end{equation}
    The second term in \eqref{eq:lma-ide-proof-2} can be bounded by smoothness and convexity:
    \begin{equation}
        \label{eq:lma-ide-proof-4}
        \ecn{\nabla f_m (\hat{x}_k, z_m) - \nabla f_m (x_\ast, z_m)} \overset{\eqref{eq:bregman-dif}}{\leq} 2 L \ec{D_{f_m} (\hat{x}_k, x_\ast)}.
    \end{equation}
    Using \eqref{eq:lma-ide-proof-4} and \eqref{eq:lma-ide-proof-3} in \eqref{eq:lma-ide-proof-2} and averaging with respect to $m$,
    \begin{align}
        \label{eq:lma-ide-proof-5}
        \avemm \ecn{g_k^m} &\leq \frac{3L^2}{M} \sum_{m=1}^{M} \ecn{x_k^m - \hat{x}_k} + 6 L D_{f} (\hat{x}_k, x_\ast) + 3 \sigmaf^2 \nonumber \\
        &= 3 L^2 \ec{V_k} + 6 L \ec{D_{f} (\hat{x}_k, x_\ast)} + 3 \sigmaf^2.
    \end{align}
    Using \eqref{eq:lma-ide-proof-5} in \eqref{eq:lma-ide-proof-2},
    \begin{align*}
        \sum_{t=t_p}^{v} \ec{V_t} &\leq \gamma^2 \br{H-1}^2 \sum_{k=t_p}^{v} \ec{3 L^2 V_k + 6 L D_{f} (\hat{x}_k, x_\ast) + 3 \sigmaf^2}.
    \end{align*}
    Noticing that the sum $\sum_{t=t_p}^{v} \ec{V_t}$ appears in both sides, we can rearrange
    \begin{align*}
        \br{1 - 3 \gamma^2 \br{H-1}^2 L^2} \sum_{t=t_p}^{v} \ec{V_t} \leq 6 L \gamma^2 \br{H-1}^2 \sum_{k=t_p}^{v} \ec{D_{f} (\hat{x}_k, x_\ast)} + 3 \gamma^2 \br{H-1}^2 \sum_{k=t_p}^{v} \sigmaf^2.
    \end{align*}
    Finally using that our choice $\gamma$ implies that $1 - 3 \gamma^2 \br{H-1}^2 L^2 \geq \frac{3}{4}$ we have,
    \begin{align*}
        \sum_{t=t_p}^{v} \ec{V_t} \leq 8 L \gamma^2 \br{H-1}^2 \sum_{k=t_p}^{v} \ec{D_{f} (\hat{x}_k, x_\ast)} + 4 \gamma^2 \br{H-1}^2 \sum_{k=t_p}^{v} \sigmaf^2.
    \end{align*}
\end{proof}

\begin{lemma}[Optimality gap single recursion] 
    \label{lemma:optimality-gap-single-recursion}
    Suppose that Assumptions~\ref{asm:convexity-and-smoothness} and \ref{asm:finite-sum-stochastic-gradients} hold for Algorithm~\ref{alg:local_sgd} with heterogeneous data and with $M \geq 2$. Then for any $\gamma \geq 0$ we have for expectation conditional on $x_t^1, x_t^2, \ldots, x_t^m$,
    \begin{equation}
        \label{eq:9f8gff}
        \ecn{r_{t+1}} \leq \br{1 - \gamma \mu} \sqn{r_t} + \gamma L \br{1 + 2 \gamma L} V_t - 2 \gamma \br{1 - 4 \gamma L} D_{f} (\hat{x}_t, x_\ast) + \frac{4 \gamma^2 \sigmaf^2}{M},
     \end{equation}   
    where $r_{t} \eqdef \hat{x}_t - x_\ast$. In particular, if $\gamma \leq \frac{1}{8L}$, then
    \begin{equation}
        \ecn{r_{t+1}} \leq \br{1 - \gamma \mu} \sqn{r_t} + \frac{5}{4} \gamma L V_t - \frac{\gamma}{2} D_{f} (\hat{x}_t, x_\ast) + \frac{4 \gamma^2 \sigmaf^2}{M},
    \end{equation}
\end{lemma}
\begin{proof}
    First note that $\hat{x}_{t+1} = \hat{x}_t - \gamma g_t$ is always true (regardless of whether or not synchronization happens), hence
    \begin{align*}
        \sqn{\hat{x}_{t+1} - x_\ast} &= \sqn{\hat{x}_t - \gamma g_t - x_\ast} \\
        &= \sqn{\hat{x}_t - x_\ast} + \gamma^2 \sqn{g_t} - 2 \gamma \ev{\hat{x}_t - x_\ast, g_t} \\
        &= \sqn{\hat{x}_t - x_\ast} + \gamma^2 \sqn{g_t} - \frac{2 \gamma}{M} \sum_{m=1}^{M} \ev{\hat{x}_t - x_\ast, g_t^m}.
    \end{align*}
    Let $r_{t} = \hat{x}_t - x_\ast$, taking conditional expectations then using Lemmas~\ref{lemma:average-gradient-bound} and \ref{lemma:inner-product-bound},
    \begin{align*}
        \ecn{r_{t+1}} &\leq \sqn{r_t} + \gamma^2 \ecn{g_t} - \frac{2 \gamma}{M} \sum_{m=1}^{M} \ev{\hat{x}_t - x_\ast, \nabla f_m (x_t^m)} \\
        &\overset{\eqref{eq:lma-average-gradient-bound}}{\leq} \sqn{r_t} + \gamma^2 \br{ 2 L^2 V_t + 8 L D_{f} (\hat{x}_t, x_\ast) + \frac{4 \sigmaf^2}{M}} - \frac{2 \gamma}{M} \sum_{m=1}^{M} \ev{\hat{x}_t - x_\ast, \nabla f_m (x_t^m)}  \\
        &\overset{\eqref{eq:lma-inner-product-bound}}{\leq} \br{1 - \gamma \mu} \sqn{r_t} + \gamma L \br{1 + 2 \gamma L} V_t - 2 \gamma \br{1 - 4 \gamma L} D_{f} (\hat{x}_t, x_\ast) + \frac{4 \gamma^2 \sigmaf^2}{M}.
    \end{align*}
    If $\gamma \leq \frac{1}{8L}$, then $1 - 4 \gamma L \geq \frac{1}{2}$ and $1 + 2 \gamma L \leq \frac{5}{4}$, and hence
    \begin{align*}
        \ecn{r_{t+1}} &\leq \br{1 - \gamma \mu} \sqn{r_t} + \frac{5}{4} \gamma L V_t - \frac{\gamma}{2} D_{f} (\hat{x}_t, x_\ast) + \frac{4 \gamma^2 \sigmaf^2}{M},
    \end{align*}
    as claimed.
\end{proof}

\subsection{Proof of Theorem~\ref{thm:wc-noniid-unbounded-var}}
\begin{proof}
    Start with Lemma~\ref{lemma:optimality-gap-single-recursion} with $\mu = 0$,
    \begin{align*}
        \ecn{r_{t+1}} &\leq \sqn{r_t} + \frac{\gamma}{2} \br{\frac{5}{2} L V_t - D_{f} (\hat{x}_t, x_\ast)} + \frac{4 \gamma^2 \sigmaf^2}{M}.
    \end{align*}
    Taking unconditional expectations and summing up,
    \begin{align}
        \label{eq:thm-lsgd-dd-proof-1}
        \sum_{i=1}^{T} \ecn{r_t} &\leq \sum_{i=0}^{T-1} \ecn{r_t} + \frac{\gamma}{2} \sum_{i=0}^{T-1} \ec{\frac{5}{2} L V_i - D_{f} (\hat{x}_i, x_\ast)} + \sum_{i=0}^{T-1} \frac{4 \gamma^2 \sigmaf^2}{M}.
    \end{align}
    Using that $T = t_p$ for some $p \in \N$, we can decompose the second term by double counting and bound it by Lemma~\ref{lemma:iterate-deviation-epoch},
    \begin{align*}
        \sum_{i=0}^{T-1} \ec{\frac{5}{2} L V_i - D_{f} (\hat{x}_i, x_\ast)} &= \sum_{k=1}^{p} \sum_{i=t_{k-1}}^{t_{k} - 1} \ec{\frac{5}{2} L V_i - D_{f} (\hat{x}_i, x_\ast)} \\
        &\leq \sum_{k=1}^{p} \sum_{i=t_{k-1}}^{t_k - 1} \br{20 L^2 \gamma^2 (H-1)^2 - 1} \ec{D_f (\hat{x}_i, x_\ast} + \sum_{k=1}^{p} \sum_{i=t_{k-1}}^{t_k - 1} 10 L \gamma^2 (H-1)^2 \sigmaf^2.
    \end{align*}
    By assumption on $\gamma$ we have that $20 L^2 \gamma^2 (H-1)^2 - 1 \leq -\frac{1}{2}$, using this and then using double counting again we have,
    \begin{align*}
        \sum_{i=0}^{T-1} \ec{\frac{5}{2} L V_i - D_{f} (\hat{x}_i, x_\ast)} &\leq -\frac{1}{2} \sum_{k=1}^{p} \sum_{i=t_{k-1}}^{t_k - 1} \ec{D_f (\hat{x}_i, x_\ast)} + \sum_{k=1}^{p} \sum_{i=t_{k-1}}^{t_k - 1} 10 L \gamma^2 (H-1)^2 \sigmaf^2 \\
        &= -\frac{1}{2} \sum_{i=0}^{T-1} \ec{D_f (\hat{x}_i, x_\ast)} + \sum_{i=0}^{T-1} 10 L \gamma^2 (H-1)^2 \sigmaf^2.
    \end{align*}
    Using this in \eqref{eq:thm-lsgd-dd-proof-1},
    \begin{align*}
        \sum_{i=1}^{T} \ecn{r_t} &\leq \sum_{i=0}^{T-1} \ecn{r_t} - \frac{\gamma}{4} \sum_{i=0}^{T-1} \ec{D_{f} (\hat{x}_i, x_\ast)} + \sum_{i=0}^{T-1} \br{5 L \gamma^3 (H-1)^2 \sigmaf^2 + \frac{4 \gamma^2 \sigmaf^2}{M}}.
    \end{align*}
    Rearranging, we get
    \begin{align*}
        \frac{\gamma}{4} \sum_{i=0}^{T-1} \ec{D_f (\hat{x}_i, x_\ast)} &\leq \sum_{i=0}^{T-1} \ecn{r_t} - \sum_{i=1}^{T} \ecn{r_t} + \sum_{i=0}^{T-1} \br{5 L \gamma^3 (H-1)^2 \sigmaf^2 + \frac{4 \gamma^2 \sigmaf^2}{M}} \\
        &= \sqn{r_0} - \ecn{r_T} + \sum_{i=0}^{T-1} \br{5 L \gamma^3 (H-1)^2 \sigmaf^2 + \frac{4 \gamma^2 \sigmaf^2}{M}} \\
        &\leq \sqn{r_0} + T \br{5 L \gamma^3 (H-1)^2 \sigmaf^2 + \frac{4 \gamma^2 \sigmaf^2}{M}}.
    \end{align*}
    Dividing both sides by $\gamma T / 4$, we get
    \begin{align*}
        \frac{1}{T} \sum_{i=0}^{T-1} \ec{D_f (\hat{x}_i, x_\ast)} &\leq \frac{4 \sqn{r_0}}{\gamma T} + \frac{20 \gamma \sigmaf^2}{M} + 16 \gamma^2 L (H-1)^2 \sigmaf^2.
    \end{align*}
    Finally, using Jensen's inequality and the convexity of $f$ we get the required claim.
\end{proof}

\section{Extra Experiments}
Figure~\ref{fig:a5a_same_data_001} shows experiments done with identical data and Figure~\ref{fig:mushrooms_different_H} shows experiments done with heterogeneous data in the same setting as described in the main text but with different datasets.

\begin{figure}
	\centering
	\includegraphics[scale=0.4]{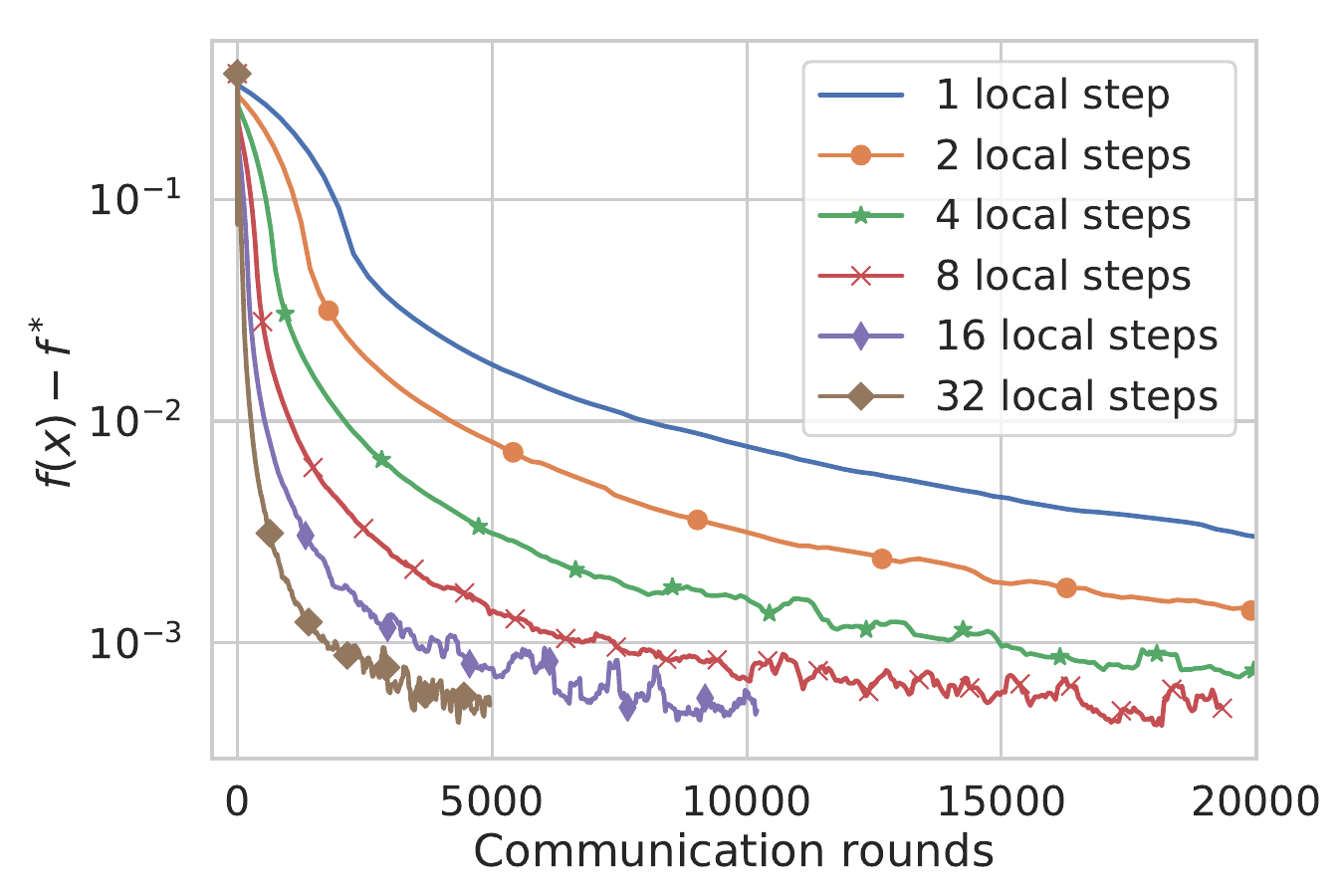}
	\includegraphics[scale=0.4]{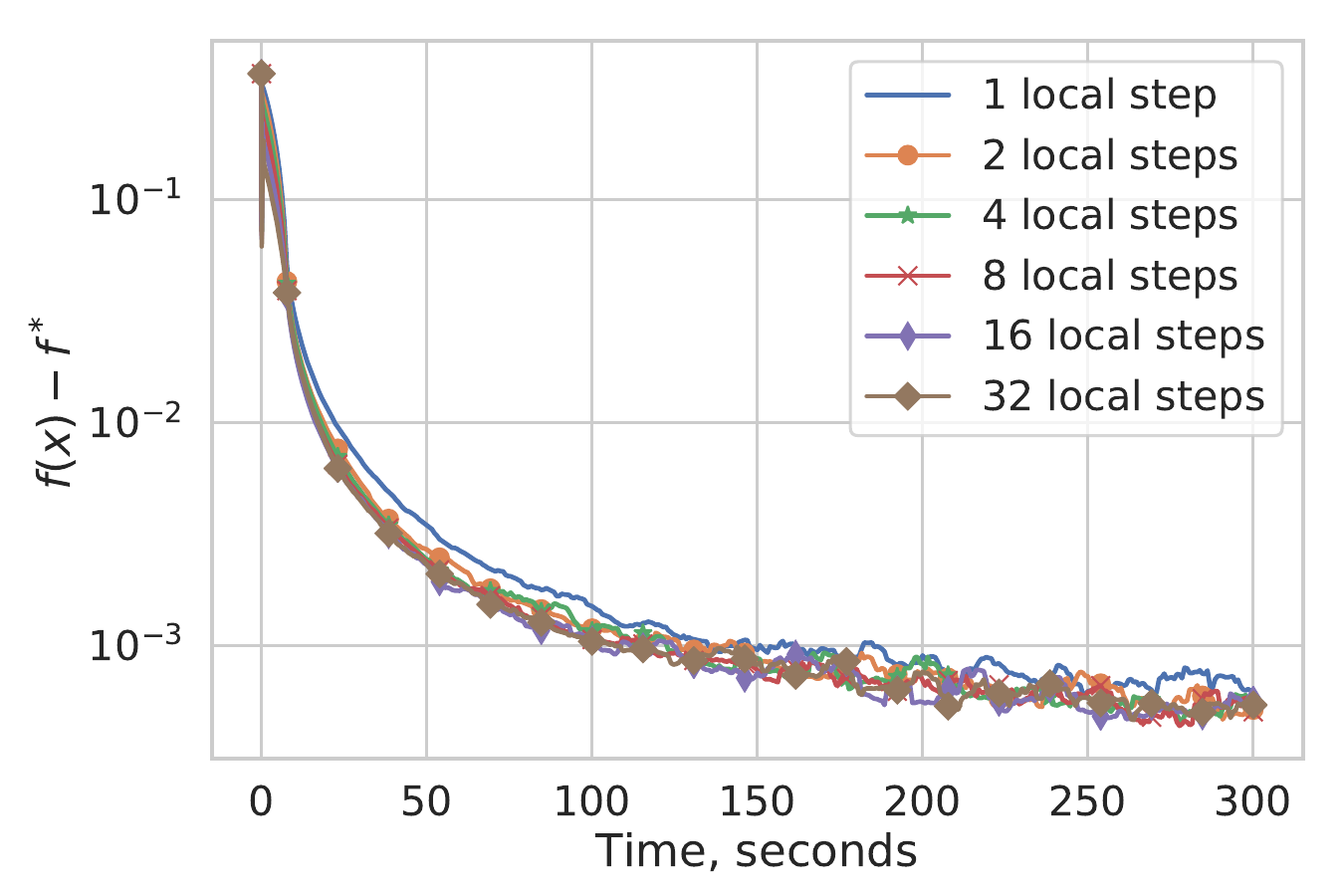}
	\caption{Results on regularized logistic regression with shared data, `a9a' dataset, with stepsize $\frac{0.05}{L}$. With more local iterations, fewer communication rounds are required to get to a neighborhood of the solution.}
	\label{fig:a5a_same_data_001}
\end{figure}

\begin{figure}[!b]
	\centering
	\includegraphics[scale=0.4]{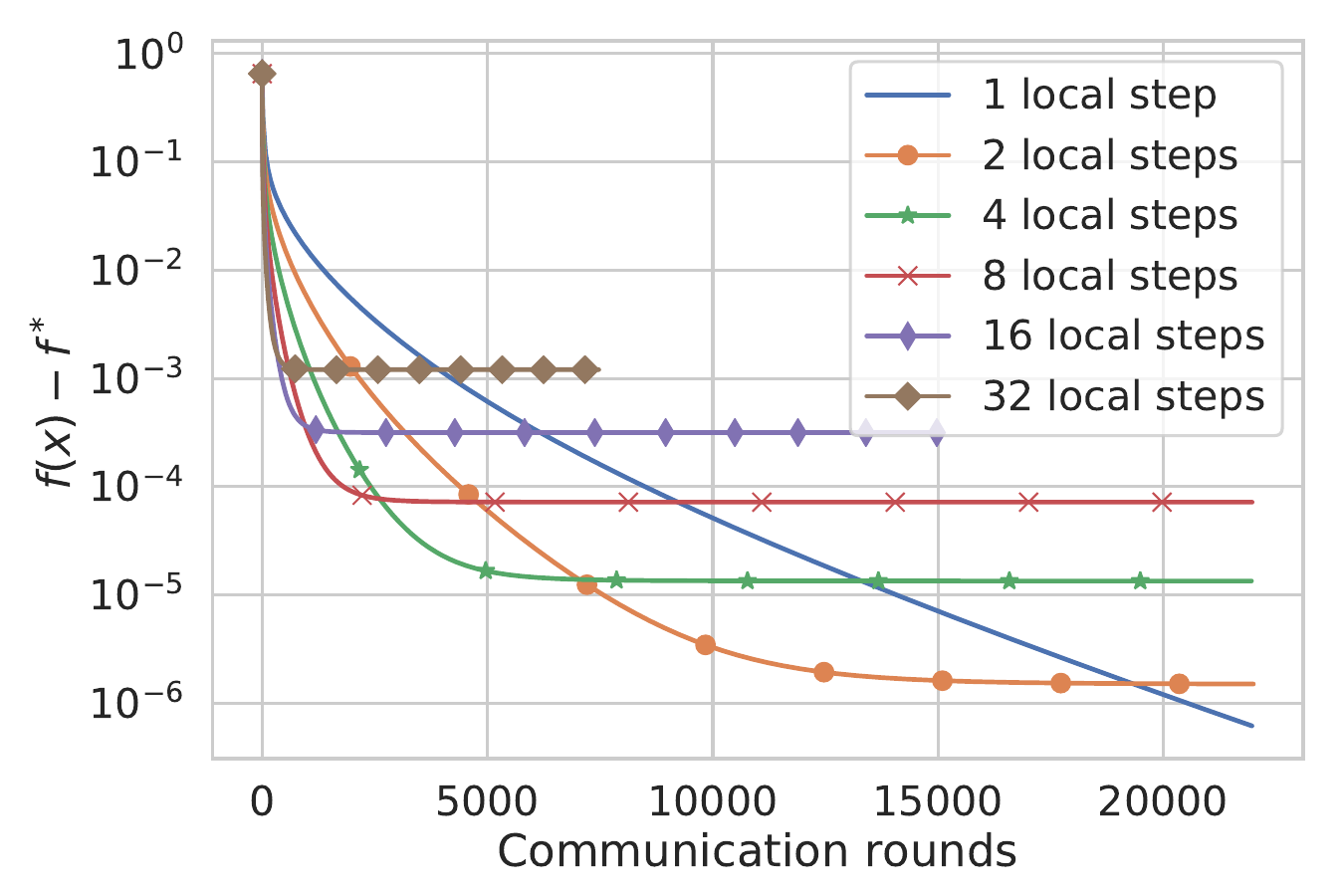}
	\includegraphics[scale=0.4]{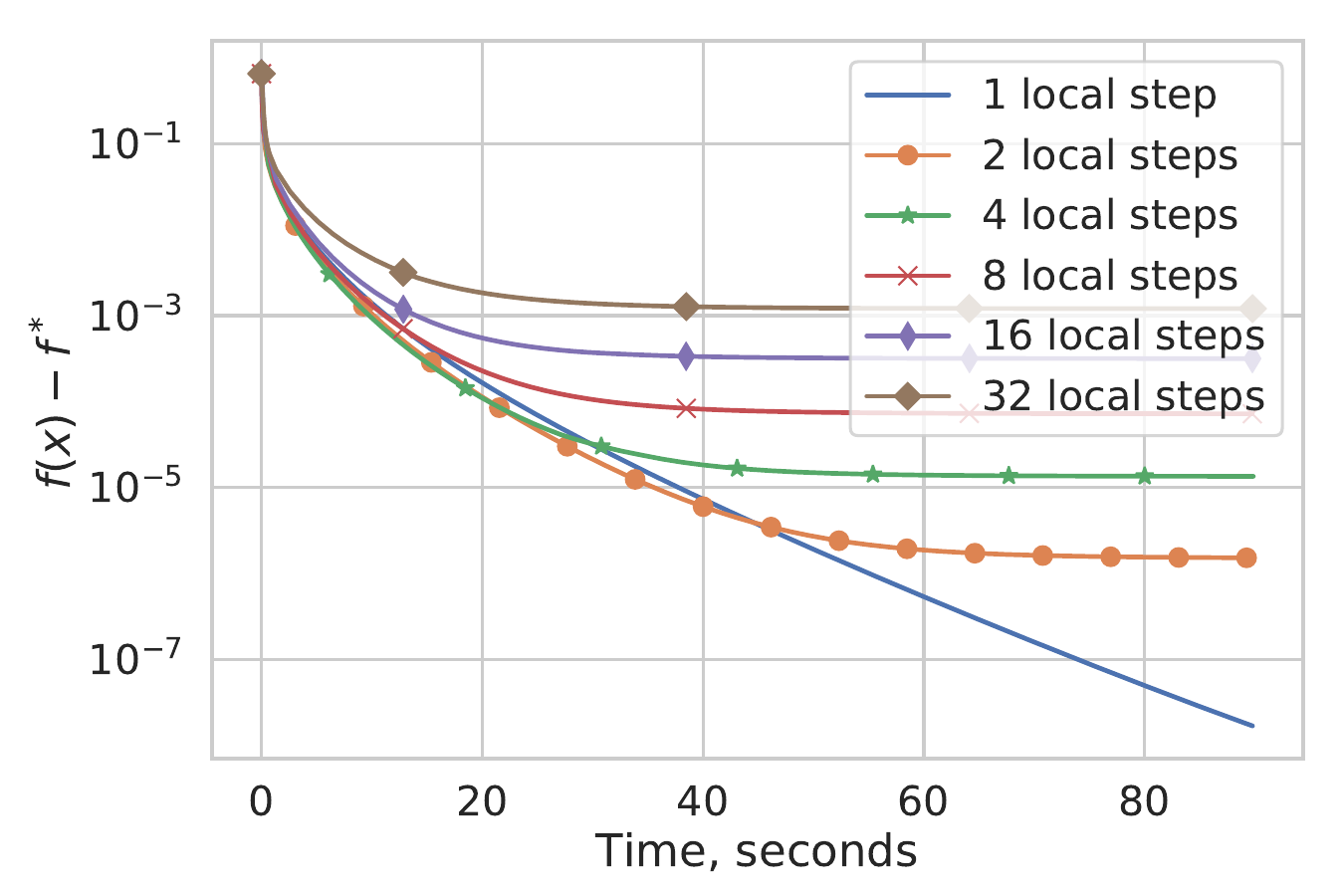}
	\includegraphics[scale=0.4]{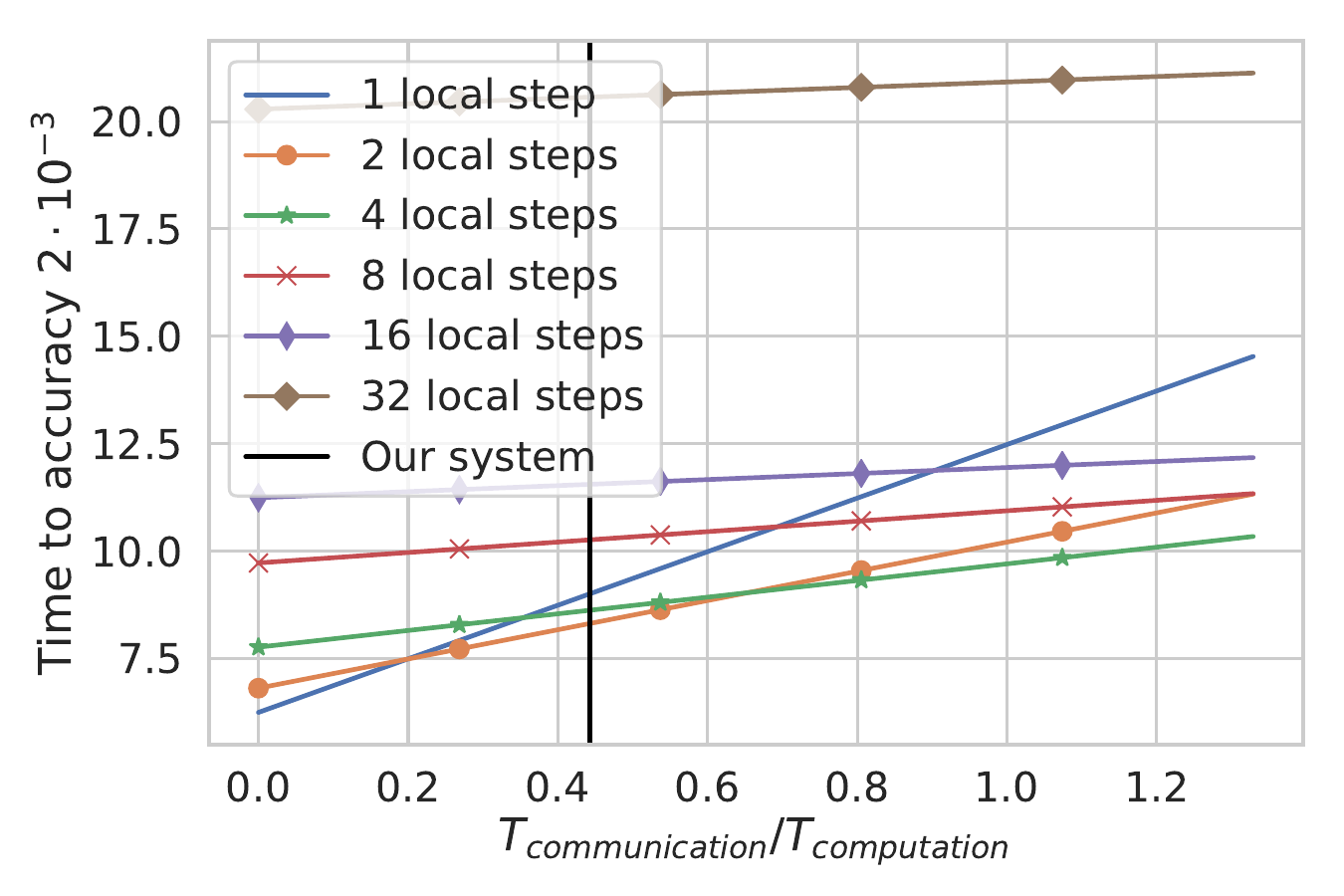}
	\caption{Same experiment as in Figure~\ref{fig:a5a_different_H}, performed on the `mushrooms' dataset.} 
	\label{fig:mushrooms_different_H}
\end{figure}

\section{Discussion of Dieuleveut and Patel (2019)}
\label{sec:patel-discuss}
An analysis of Local SGD for identical data under strong convexity, Lipschitzness of $\nabla f$, uniformly bounded variance, and Lipschitzness of $\nabla^2 f$ is given in \citep{Patel19}, where they obtain a similar communication complexity to \citep{Stich2018} without bounded gradients. However, in the proof of their result for general non-quadratic functions (Proposition S20) they make the following assumption, rewritten in our notation:
\[ G = \sup_{p} \br{ 1 + M L_H \gamma \sum_{k=t_{p}}^{t_{p+1} - 1} \sqn{\hat{x}_k - x_\ast} } < \infty,  \]
where $L_H$ is the Lipschitz constant of the Hessian of $f$ (assumed thrice differentiable). Their discussion of $G$ speculates on the behaviour of iterate distances, e.g.\ saying that \emph{if} they are bounded, then the guarantee is good. Unfortunately, assuming this quantity bounded implies that gradients are bounded as well, making the improvement over \citep{Stich2018} unclear to us. Furthermore, as $G$ depends on the algorithm's convergence (it is the distance from the optimum evaluated at various points), assuming it is bounded to prove convergence to a compact set results in a possibly circular argument. Since $G$ is also used as an upper bound on $H$ in their analysis, it is not possible to calculate the communication complexity.

\end{document}

%% file: algs/gen_loc_sgd_alg.tex
\begin{algorithm*}[t]
   \caption{Local SGD}
   \label{alg:local_sgd}
\begin{algorithmic}[1]
  \REQUIRE Stepsize $\gamma > 0$, initial vector $x_0 = x_0^m$ for all $m \in [M]$, synchronization timesteps $t_1, t_2, \ldots$.
   \FOR{$t=0,1,\dotsc$}
      \FOR{$m=1,\dotsc, M$ in parallel}
         \STATE Sample $z_m \overset{\text{i.i.d.}}{\sim} \D_m$.
         \IF{data is identical}
            \STATE Compute $g_t^m = g(f, x_t^m, z_m)$ such that $\ec{g_t^m \mid x_t^m} = \nabla f(x_t^m)$.
         \ELSE
            \STATE Compute $g_t^m = g(f_m, x_t^m, z_m)$ such that $\ec{g_t^m\mid x_t^m}=\nabla f_m (x_t^m)$.
         \ENDIF
         \STATE $x_{t+1}^m=
         \begin{cases}
         \frac{1}{m}\sum_{j=1}^M (x_t^j - \gamma g_t^j), & \text{ if } t = t_p \text { for some } p \in \N \\
         x_t^m - \gamma g_t^m, & \text{ otherwise. }
         \end{cases}$
      \ENDFOR
   \ENDFOR
\end{algorithmic}
\end{algorithm*}